%% file: stability.tex
\documentclass[journal]{IEEEtran}

\usepackage[english]{babel}

\usepackage{ifpdf}

\usepackage{cite} 
\usepackage{url}
\usepackage{hyperref}

\ifCLASSINFOpdf
	\usepackage[pdftex]{graphicx}
	\graphicspath{{./figures/}}
\else
	\usepackage[dvips]{graphicx}
	\graphicspath{./figures/}
\fi
\usepackage{color}
\usepackage{pgf, tikz, pgfplots}
\usetikzlibrary{shapes, arrows, automata, plotmarks}
\usetikzlibrary{calc,hobby,decorations}

\usepackage[cmex10]{amsmath}
\usepackage{amsfonts, amssymb, amsthm}
\usepackage{mathrsfs}

\usepackage{algorithm,algpseudocode}
	\algnewcommand{\LeftComment}[1]{\Statex \(\triangleright\) #1}

\usepackage{enumerate}
\usepackage{multirow}
\usepackage{rotating}
\usepackage{subcaption}
\input{mySections.tex}

\input{mySymbol.sty}
\input{pennColors.sty}

\def \hlam {\hat{\lam}}

\def\Tr{\mathsf{T}}
\def\Hr{\mathsf{H}}

\def\eps{\varepsilon}

\newtheorem{lemma}{\hspace{0pt}\bf Lemma}
\newtheorem{proposition}{\hspace{0pt}\bf Proposition}

\newtheorem{theorem}{\hspace{0pt}\bf Theorem}
\newtheorem{corollary}{\hspace{0pt}\bf Corollary}

\newtheorem{remark}{\hspace{0pt}\bf Remark}

\newtheorem{definition}{\hspace{0pt}\bf Definition}

\begin{document}

\title{Stability Properties of Graph Neural Networks}

\author{Fer\hspace{0.015cm}nando~Gama,~
        Joan~Bruna,~
        and~Alejandro~Ribeiro
\thanks{Fernando Gama and Alejandro Ribeiro are supported by NSF CCF 1717120, ARO W911NF1710438, ARL DCIST CRA W911NF-17-2-0181, ISTC-WAS and Intel DevCloud. Joan Bruna is partially supported by the Alfred P. Sloan Foundation, NSF RI-1816753, NSF CAREER CIF 1845360, and Samsung Electronics. F. Gama is with the Elect. Eng. Comput. Sci. Dept., Univ. California, Berkeley, J. Bruna is with the Courant Inst. Math. Sci. and Center Data Sci., New York Univ. and A. Ribeiro is with the Dept. Elect. Syst. Eng., Univ. Pennsylvania.  Email: fgama@berkeley.edu, bruna@cims.nyu.edu, and aribeiro@seas.upenn.edu.
}
}

\markboth{IEEE TRANSACTIONS ON SIGNAL PROCESSING (ACCEPTED)}%
{Stability Properties of Graph Neural Networks}

\maketitle

\begin{abstract}
Graph neural networks (GNNs) have emerged as a powerful tool for nonlinear processing of graph signals, exhibiting success in recommender systems, power outage prediction, and motion planning, among others. GNNs consist of a cascade of layers, each of which applies a graph convolution, followed by a pointwise nonlinearity. In this work, we study the impact that changes in the underlying topology have on the output of the GNN. First, we show that GNNs are permutation equivariant, which implies that they effectively exploit internal symmetries of the underlying topology. Then, we prove that graph convolutions with integral Lipschitz filters, in combination with the frequency mixing effect of the corresponding nonlinearities, yields an architecture that is both stable to small changes in the underlying topology and discriminative of information located at high frequencies. These are two properties that cannot simultaneously hold when using only linear graph filters, which are either discriminative or stable, thus explaining the superior performance of GNNs.
\end{abstract}


\IEEEpeerreviewmaketitle


\input{introductionStability.tex}


\input{graphFilters.tex}


\input{GNNstability.tex}




\input{discussions.tex}


\input{experimentsStability.tex}





\input{conclusionsStability.tex}



\appendices 

\input{proofsStability.tex}


\bibliographystyle{IEEEtran}
\bibliography{myIEEEabrv,biblioStability}

\end{document}

%% file: mySections.tex
\usepackage{needspace}

\newcommand{\myparagraph}[1]{\needspace{1\baselineskip}\medskip\noindent {\bf #1}}



%% file: introductionStability.tex


\section{Introduction} \label{sec:intro}

Convolutional neural networks (CNNs) \cite{LeCun15-DeepLearning} are the tool of choice for machine learning in Euclidean space. CNNs consist of nonlinear maps in which the output follows from sending the input through a cascade of layers, each of which computes a convolution with a bank of filters followed by a pointwise nonlinearity \cite[Ch. 9]{Goodfellow16-DeepLearning}. The value of the filter taps in the convolution is obtained by minimizing some cost function over a training set. The success of CNNs is simultaneously predictable and surprising. If we were to restrict attention to linear machine learning, a century of empirical and theoretical evidence would prescribe the use of convolutional filters. It is then predictable that the addition of a pointwise nonlinearity, which on the face of it is a pretty minor modification, is a sensible choice for a nonlinear processing architecture. But at the same time it is surprising that such a minor modification produces so much of an effect on empirically observed performance. 

An answer to this question was put forth in \cite{Mallat12-Scattering} in the form of stability to diffeomorphisms. This seminal work considers \emph{scattering transforms} which are information processing architectures akin to CNNs. They are also built from convolutions and nonlinearities but use pre-specified families of wavelet frames instead of learnable filter banks. It was proved in \cite{Mallat12-Scattering} that scattering transforms are Lipschitz stable with respect to smooth deformations of space (i.e., they are stable with respect to the gradient of the diffeomorphism of the domain). This stability property is shared by linear wavelet banks, only if their frequency responses are flat at high frequencies \cite{Bruna13-Scattering}. One can restrict attention to this class of filters but only at the cost of losing the ability to discriminate high frequency features. The nonlinear operation in the scattering transform acts as a frequency mixer that brings part of the high frequency energy towards low frequencies where it can be discriminated with stable filters. Thus, scattering transforms can be, both, stable and discriminative, but wavelet banks cannot be simultaneously stable and discriminative. The similarities between scattering transforms and CNNs (both computed as convolutions followed by nonlinearities) suggest that this conclusion can be extrapolated to CNNs, in the sense that we can argue that they improve over linear filters because they are simultaneously stable and discriminative. Recent developments in computer vision similarly link the frequency content of images (Euclidean data) with their stability \cite{Yin19-FourierComputerVision, Wang20-HighFrequency}.

Parallel to the development of CNNs, the field of graph signal processing (GSP) has emerged as a generalization of Euclidean signal processing to signals whose components are related by arbitrary pairwise relationships described by an underlying graph support \cite{Shuman13-SPG, Sandryhaila14-Freq}. Central to GSP is the generalization of linear convolutional filters as polynomials of some matrix representation of the graph \cite{Segarra17-Linear}. Having a valid convolution operation the notion of a graph neural network (GNN) emerges naturally as a cascade of layers, where each layer is made up of a graph convolution filter bank composed with a pointwise nonlinearity \cite{Bruna14-DeepSpectralNetworks, Defferrard17-CNNGraphs, Kipf17-ClassifGCN, Gama18-Architectures}. GNNs have, predictably, proved useful in a variety of problems where they, surprisingly, outperform linear graph filters \cite{Ying18-Recommender, Owerko18-Power, Owerko20-OPF, Tolstaya19-Flocking, Li20-Planning}.

The main contribution of this paper is to show that the advantage of GNNs relative to linear graph filters is their stability to graph deformations (Thm.~\ref{thm:GNNStability}). Our analysis utilizes graph Fourier transforms to provide a representation of the filter on the spectrum of the matrix representation of the graph. This representation shows that graph filters cannot be stable if they are designed to isolate features associated with large eigenvalues of the graph (Thm.~\ref{thm:filterStabilityRelative}). This is the equivalent of Euclidean convolutional filters being unable to be stable if they discriminate high frequencies. The graph filter banks that are used by GNNs cannot discriminate features associated with large eigenvalues either. But pointwise nonlinearities perform frequency mixing that brings part of the energy associated with large eigenvalues towards low eigenvalues where it can be discriminated with stable graph filters. Thus, GNNs can be, both, stable and discriminative, but linear graph filters cannot be simultaneously stable and discriminative. This is analogous to the reasons that explain why scattering transforms (and, by extension, CNNs) have an advantage with respect to linear Euclidean filters \cite{Mallat12-Scattering, Bruna13-Scattering, Yin19-FourierComputerVision, Wang20-HighFrequency}.

The analysis of stability properties for the case of non-trainable GNNs built with graph wavelet filter banks has been carried out by \cite{ZouLerman18-Scattering, Gama19-Scattering}, in analogy to \cite{Mallat12-Scattering, Bruna13-Scattering}. More specifically, \cite{ZouLerman18-Scattering} studies the stability of these GNNs to permutations, as well as to perturbations on the eigenvalues and eigenvectors of the underlying graph support. Alternatively, \cite{Gama19-Scattering} considers the specific case of using diffusion wavelets and proves permutation invariance as well as stability to perturbations measured by the diffusion distance \cite{Coifman06-DiffusionDistance}. Both of these works consider an absolute perturbation model, where changes in the underlying graph support do not take into account the particularities of the topology. This leads to results that either depend on the size of the graph (i.e. larger graphs admit smaller edge weight changes) \cite{ZouLerman18-Scattering} or on the spectral gap \cite{Gama19-Scattering}, tying the applicability of the results to the specific graph under consideration. Stability of arbitrary filter banks has been studied in \cite{Levie19-Transferability} by leveraging a bound in \cite[eq. (23)]{Gama19-Scattering}. This result also depends on the spectral gap. Permutation equivariance has been studied in \cite{Xu19-GIN, Maron19-Invariant, Keriven19-UniversalInvariant}. In particular, \cite{Xu19-GIN} considers the question of graph isomorphisms by means of the Weisfeiler-Lehman test, \cite{Maron19-Invariant} characterizes the space of invariant and equivariant linear layers, and \cite{Keriven19-UniversalInvariant} extends the result in \cite{Maron19-Invariant} to graphs of varying size.

The underlying support can change due to targeted attacks on the nodes and edges of the graph. The robustness of GNNs to these malicious, adversarial attacks is being studied. In \cite{Zugner18-AdversarialAttacks}, the focus is on designing adversarial attacks such that the label of a target node is misclassified by carefully changing the edges and signal values of other nodes. It considers binary adjacency matrices and binary graph signals, and assumes that the semi-supervised problem is solved by means of a GCN \cite{Kipf17-ClassifGCN} with a single-hidden layer. The work in \cite{Dai18-AttackGraphData} also focuses on designing adversarial attacks, but uses reinforcement learning and focuses on the problems of both node and graph classification. In the case when the attacks are crafted to focus on a subset of edges in a semi-supervised learning problem, where labels are inferred using GNNs with IIR filters on binary adjacency matrices, \cite{Bojchevski19-CertifiableRobustness} provides robustness certificates for which nodes will not change the learned label under these attacks, and also proposes robust training of the model by adding a penalty to the cross-entropy loss function. Robustness certificates and robust training have also been developed for adversarial attacks on the binary signal values of a semi-supervised learning problem, and where labels are obtained by means of a GCN \cite{Zugner19-CertifiableSignals}. In this paper, however, we focus on changes that can occur due to topology inference errors or due to time-varying scenarios, instead of crafted attacks.

We begin the paper in Sec.~\ref{sec:graphFilters} by showing that linear graph filters are equivariant to permutations (Prop.~\ref{prop:filterPermutationEquivariance}). Then, we discuss the model of absolute perturbations modulo permutation (analogous to that in \cite{ZouLerman18-Scattering, Gama19-Scattering}) and show that a linear filter whose frequency response is Lipschitz continuous is stable (Thm.~\ref{thm:filterStabilityAbsolute}), with a constant that depends on the Lipschitz condition of the filters as well as the intrinsic topological characteristics of the graph and its perturbation. Next, we show that, under the relative perturbation model, integral Lipschitz graph filters are stable (Thm.~\ref{thm:filterStabilityRelative}), and determine a family of perturbations under which the stability can be entirely controlled by the integral Lipschitz constant of the filters, for any graph (Thm.~\ref{thm_filter_stability_structural_constraint}). In Sec.~\ref{sec:graphNeuralNetworks} we show how the stability results for graph filters carry over to GNN architectures (Thm.~\ref{thm:GNNStability}). Sec.~\ref{sec_linear_filters_stability} offers an insightful discussion on the results, showing that a GNN built on Lipschitz filters under an absolute perturbation model exhibits a trade-off between stability and discriminability, while another one using integral Lipschitz filters under a relative perturbation model can be made simultaneously stable and discriminative. Finally, we numerically illustrate stability in a problem of movie recommendation (Sec.~\ref{sec:sims}), and conclude (Sec.~\ref{sec:conclusions}).

%% file: graphFilters.tex


\section{Stability Properties of Graph Filters} \label{sec:graphFilters}

We work with graphs $\ccalG = (\ccalV, \ccalE, \ccalW)$ described by a set of $N$ nodes $\ccalV$, a set of edges $\ccalE \subseteq \ccalV \times \ccalV$, and a weight function $\ccalW:\ccalE \to \reals$. We can associate to $\ccalG$ a matrix representation $\bbS \in \reals^{N \times N}$ that respects the sparsity of the graph, namely, $s_{ij} = [\bbS]_{ij} = 0$ whenever $i \neq j$ or $(j,i) \notin \ccalE$. This is a condition that is verified by, e.g., adjacency matrices, Laplacians, random walk Laplacians, and their normalized counterparts. We generically call $\bbS$ a graph shift operator (GSO) \cite{Shuman13-SPG}. We assume the shift operator is symmetric with eigenvector basis $\bbV = [\bbv_1,\ldots,\bbv_N]$ and eigenvalue matrix $\bbLam = \diag ([\lam_1,\ldots,\lam_N])$ so that we can write
\begin{equation} \label{eqn:eigendecomposition}
   \bbS = \bbV\bbLam\bbV^{\Hr}.
\end{equation}
It is assumed that eigenvalues are ordered from smallest to largest so that $\lam_1\leq\lam_2\leq\ldots\leq\lam_N$.

The graph acts as a support for the data vector $\bbx \in \reals^{N}$ which we henceforth say to be a graph signal $\bbx = [x_{1}, \ldots, x_{N}]^{\Tr}$ that assigns the value $x_i$ to node $i$. The shift operator $\bbS$ defines a linear map $\bby = \bbS \bbx$ between graph signals that represents the local exchange of information between a node and its neighbors. Repeated application of $\bbS$ accesses information from nodes located farther away since the product $\bbS^{k} \bbx = \bbS (\bbS^{k-1}\bbx)$ represents the aggregation at node $i$ of information located in nodes of its $k$-hop neighborhood. Aggregating information from $k$-hop neighbors is analogous to applying $k$ time shifts to a time signal. This is the motivation for defining graph convolutional filters as polynomials on the shift operator that, for a set of coefficients $\bbh = \{h_{k}\}_{k=0}^\infty$ process input graph signals $\bbx$ to produce output graph signals,
\begin{equation} \label{eqn:graphFilter}
	\bbz \  = \ \sum_{k=0}^{\infty} h_{k} \bbS^{k} \bbx
	     \ := \ \bbH(\bbS) \bbx .
\end{equation}
The matrix $\bbH(\bbS) := \sum_{k=0}^{\infty} h_{k} \bbS^{k}$ in \eqref{eqn:graphFilter} is said to be a finite impulse response (FIR) graph filter or a graph convolutional filter, and the coefficients $h_{k}$ are thus called \emph{filter taps} or \emph{filter weights} \cite{Segarra17-Linear}. A set of coefficients $\bbh$, defines a filter $\bbH(\bbS)$ for any given graph $\bbS$. In particular, if we are given another graph $\hbS$ we can construct the filter $\bbH(\hbS):=\sum_{k=0}^{\infty} h_{k} \hbS^{k}$, with the same filter taps $\bbh$, whose application to graph signals $\hbx$ produces graph signals
\begin{equation} \label{eqn:graphFilter_perturbed}
	\hbz \  = \ \sum_{k=0}^{\infty} h_{k} \hbS^{k} \hbx
	     \ := \ \bbH(\hbS) \hbx.
\end{equation}
We want to characterize the difference between filters $\bbH(\bbS)$ and $\bbH(\hbS)$ in terms of the differences between shift operators $\bbS$ and $\hbS$. We study the effect of permutations in Sec.~\ref{sec_permutation_equivariance} and the effect of perturbations in Secs.~\ref{sec_graph_perturbations}~and~\ref{sec_graph_perturbations_relative}.

%
\subsection{Permutation Equivariance of Graph Filters}\label{sec_permutation_equivariance}

For a given dimension $N$ we define permutation matrices $\bbP$ as those that belong to the set
\begin{equation} \label{eqn:permutationSet}
    \ccalP = \left\{
        \bbP \in \{0,1\}^{N \times N}
            :
        \bbP \bbone = \bbone , 
        \bbP^{\Tr} \bbone = \bbone
    \right\}.
\end{equation}
As per \eqref{eqn:permutationSet}, a permutation matrix $\bbP$ is one in which the product $\hbx = \bbP^{\Tr}\bbx$ is a reordering of the entries of the vector $\bbx$. Likewise, for a given shift operator $\bbS$, the shift operator $\hbS = \bbP^{\Tr}\bbS\bbP$ is a reordering of the rows and columns of $\bbS$. Thus, the graph $\hbS$ is just a relabeling of the nodes of $\bbS$. Taken together, the graph $\hbS = \bbP^{\Tr}\bbS\bbP$ and the signal $\hbx = \bbP^{\Tr}\bbx$ represent a consistent relabeling of the graph $\bbS$ and the graph signal $\bbx$. One would expect that the application of the filter on the relabeled graph to the relabeled signal produces an output that corresponds to the relabeled output prior to permutation of the graph. The following proposition asserts that this is true.

%
\begin{proposition}[Permutation equivariance] \label{prop:filterPermutationEquivariance} Consider graph shifts $\bbS$ and $\hbS = \bbP^{\Tr} \bbS \bbP$ for some permutation matrix $\bbP \in \ccalP$ [cf. \eqref{eqn:permutationSet}]. Given a set of coefficients $\bbh$, the filters $\bbH(\bbS)$ and  $\bbH(\hbS)$ in \eqref{eqn:graphFilter} and \eqref{eqn:graphFilter_perturbed} are such that for any pair of corresponding graph signals $\bbx$ and $\hbx = \bbP^{\Tr}\bbx$ the graph filter outputs $\bbz := \bbH(\bbS) \bbx $ and $\hbz := \bbH(\hbS) \hbx $ satisfy
\begin{equation}\label{eqn_prop_filter_permutation_equivariance}
   \hbz := \bbH(\hbS) \hbx
        := \bbH(\hbS) ( \bbP^{\Tr}\bbx ) 
         = \bbP^{\Tr} (  \bbH(\bbS) \bbx )
        := \bbP^{\Tr} \bbz.
\end{equation} \end{proposition}

%
\begin{proof} See Appendix \ref{sec_apx_A}. \end{proof}

%
Prop.~\ref{prop:filterPermutationEquivariance} states the permutation equivariance of graph filters. Namely, a permutation of the input -- from $\bbx$ to $\hbx = \bbP^{\Tr}\bbx$ -- along with a permutation of the shift operator -- from $\bbS$ to $\hbS = \bbP^{\Tr} \bbS \bbP$ -- results in a permutation of the output -- from $\bbz$ to $\hbz = \bbP^{\Tr} \bbz$. Notice that if we are given vectors $\bbx$ and $\hbx = \bbP^{\Tr}\bbx$ and we know the permutation matrix $\bbP$ it is elementary to design linear operators that are permutation equivariant by simply applying the permutation to the linear operator. The permutation equivariance in \eqref{eqn_prop_filter_permutation_equivariance} is more subtle in that it holds without having access to the permutation $\bbP$. 

Permutation equivariance of graph filters implies their usefulness in applications where graph relabeling is inconsequential. This motivates consideration of the space of linear operators modulo permutation along with operator distances between these equivalence classes which we introduce next.

%
\begin{definition}[Linear operator distance modulo permutation] \label{def_operator_distance} 
Given linear operators $\bbA$ and $\hbA$ we define the operator distance modulo permutation as
\begin{align} \label{eqn_def_operator_distance}
   \big\| \bbA - \hbA \big\|_{\ccalP} 
      \ = \   \min_{\bbP\in\ccalP} \,
                 \max_{ \bbx : \| \bbx \| = 1 } \,
                     \big\|\bbP^{\Tr}(\bbA \bbx)  - \hbA (\bbP^{\Tr} \bbx)\big\| .
\end{align}
\end{definition}

%
The distance in \eqref{eqn_def_operator_distance} compares the effect of applying $\bbA$ and $\hbA$ to a vector $\bbx$ with a permutation $\bbP$ applied before or after application of the linear operators. It measures this difference at the unit norm vector for which it is largest and at the permutation matrix that makes it smallest. We can further denote as $\bbP_0 \in \ccalP$ a matrix that attains the minimum in \eqref{eqn_def_operator_distance} and define the \emph{absolute perturbation modulo permutation} as
\begin{align} \label{eqn_def_operator_distance_error_matrix}
\bbE = \bbA - \bbP_0^{\Tr} \hbA\bbP_0.
\end{align}
The perturbation matrix $\bbE$ is the difference between operators $\bbA$ and $\bbP_{0}^{\Tr} \hbA \bbP_0$ for the permutation matrix $\bbP_0$ that achieves the minimum norm difference in \eqref{eqn_def_operator_distance}. If there is more than one matrix that achieves the minimum, an arbitrary choice is acceptable for the results we will derive to hold. Further notice that it follows from \eqref{eqn_def_operator_distance} and \eqref{eqn_def_operator_distance_error_matrix} that the operator distance between $\bbA$ and $\hbA$ is simply the operator norm of $\bbE$,
\begin{align} \label{eqn_def_operator_distance_error}
   \big\| \bbA - \hbA \big\|_{\ccalP}  = \| \bbE \|.
\end{align}
It is ready to see that Def.~\ref{def_operator_distance} is a proper distance in the space of linear operators modulo permutation. In particular, it holds that $\big\| \bbA - \hbA \big\|_{\ccalP} = 0$ if and only there exists a permutation matrix for which $\hbA = \bbP^{\Tr} \bbA \bbP$. From this latter fact it follows that we can rewrite Prop.~\ref{prop:filterPermutationEquivariance} to say that the distance modulo permutation between graph filters $\bbH(\bbS)$ and $\bbH(\hbS)$ is null if the distance modulo permutation between the shift operators $\bbS$ and $\hbS$ is null. We formally state and prove this fact next.

%
\begin{corollary}[Permutation equivariance] \label{coro_permutation_equivariance} 
For shift operators $\bbS$ and $\hbS$ whose distance modulo permutation is $\| \bbS - \hbS \|_{\ccalP} = 0$, the distance modulo permutation between graph filters $\bbH(\bbS)$ and $\bbH(\hbS)$ satisfies 
\begin{equation}\label{eqn_coro_permutation_equivariance} 
   \| \bbH(\bbS) - \bbH(\hbS) \|_{\ccalP} = 0.
\end{equation} \end{corollary}

%
\begin{proof}
If $\| \bbS - \hbS \|_{\ccalP} = 0$ there exists permutation $\bbP$ such that $\hbS = \bbP^{\Tr} \bbS \bbP$ and Proposition \ref{prop:filterPermutationEquivariance} holds. Thus, for this same permutation we have $\bbP^{\Tr}(\bbH(\bbS) \bbx)  = \bbH(\hbS) (\bbP^{\Tr} \bbx)$ for any vector $\bbx$. The result in \eqref{eqn_coro_permutation_equivariance} then follows from Def.~\ref{def_operator_distance}.
\end{proof}

%
Corollary \ref{coro_permutation_equivariance} says that if two graphs are the same, the graph filters are also the same, modulo permutation. In the next section we address the question of what happens to the filter difference $\| \bbH(\bbS) - \bbH(\hbS) \|_{\ccalP}$ when the difference $\| \bbS - \hbS \|_{\ccalP}$ is small but not null. Namely, we investigate the similarity of $\bbH(\bbS)$ and $\bbH(\hbS)$ when the shift operators $\bbS$ and $\hbS$ are close to being permuted versions of each other.

%
\subsection{Effect of Absolute Graph Perturbations on Graph Filters}\label{sec_graph_perturbations}

To understand the stability of graph filters it is instructive to consider the form of \eqref{eqn:graphFilter} in the graph frequency domain. This entails using the eigenvector basis $\bbV$ in \eqref{eqn:eigendecomposition} to define the {\it Graph Fourier Transform} (GFT) of the graph signal $\bbx$ as the projection $\tbx = \bbV^{\Hr}\bbx$ \cite{Sandryhaila14-Freq}. Substituting the GFT definition in the definition of the graph convolution in \eqref{eqn:graphFilter} and using the fact that $\bbS^k = \bbV\bbLam^k\bbV^{\Hr}$ we can write 
\begin{equation} \label{eqn:filter_dft}
   \bbV^{\Hr}\bbz \ =\  \tbz \ 
	          \ = \ \sum_{k=0}^{\infty} h_{k}\bbLam^{k}\big(\bbV^{\Hr}\bbx\big)
	          \ = \ \bbH(\bbLam) \hbx .
\end{equation}
The interesting conclusion that follows from \eqref{eqn:filter_dft} is that graph filters are pointwise operators in the graph frequency domain because $\bbH(\bbLam)$ is a diagonal matrix. This motivates the definition of the \emph{graph frequency response} of the filter as
\begin{equation} \label{eqn:frequency_response}
	h(\lam)  \ = \ \sum_{k=0}^{\infty} h_{k} \lam^{k} .
\end{equation}
Comparing \eqref{eqn:filter_dft} with \eqref{eqn:frequency_response} we conclude that the $i$th component $\tdx_i$ of the input signal GFT $\tbx$ and  the $i$th component $\tdz_i$ of the output signal GFT $\tbz$ are related through the expression
\begin{equation} \label{eqn:frequency_response_pointwise}
	\tdz_i = h(\lam_i) \tdx_i.
\end{equation}
The remarkable observation to be made at this point is that the frequency response of a filter is completely characterized by the filter coefficients $\bbh$ [cf. \eqref{eqn:frequency_response}]. The effect of a specific graph is to determine the components of the GFT $\tbx = \bbV^{\Hr}\bbx$ -- through its eigenvectors $\bbV$ -- and which values of the frequency response are instantiated [cf. \eqref{eqn:frequency_response_pointwise}] -- through its eigenvalues $\bbLam$. Fig.~\ref{fig_frequency_response} shows an illustration of this latter fact. We have a filter with frequency response $h(\lam)$ represented as a continuous function. For a graph with eigenvalues $\lam_i$ only the values at frequencies $h(\lam_i)$ affect the response of the filter. For a different graph with eigenvalues $\hlam_i$ the values $h(\hlam_i)$ are the ones that determine the effect of the filter in the given graph. 

%
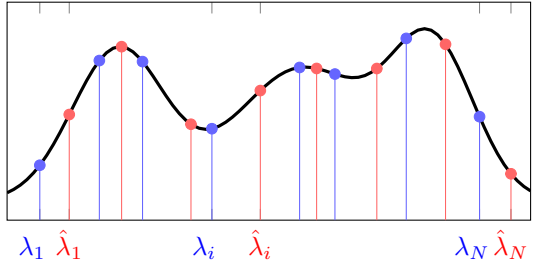
\begin{figure}[t]
    \centering
    \input{plots_stability/frequency_response_instantiation.tex}    
    \caption{Frequency response of a graph filter. The function $h(\lambda)$ is shown as a black, solid line, and is independent of the graph. When evaluated on a given graph shift operator, specific values of $h$ are instantiated on the eigenvalues of the given GSO. For example, when given a GSO $\bbS$ with eigenvalues $\{\lambda_{i}\}$, the graph filter frequency response will be instantiated on $h(\lambda_{i})$ (in blue); but, if we are given a different GSO $\hbS$ with other eigenvalues $\{\hlam_{i}\}$, then the filter frequency response will be given by $h(\hlam_{i})$ (in red).}
    \label{fig_frequency_response}
\end{figure}

Since graph perturbations alter the spectrum of a graph it seems apparent that the variability of the frequency response $h(\lam)$ has a direct effect on a filter's stability to perturbations. To proceed with a formal characterization we introduce the notion of Lipschitz filters in the following definition.
 
%
\begin{definition}[Lipschitz Filter]\label{def:lipschitzFilters}
Given a filter $\bbh = \{h_{k}\}_{k=0}^\infty$ its frequency response $h(\lam)$ is given by \eqref{eqn:frequency_response} and satisfies $|h(\lam)|\leq 1$. We say the filter is Lipschitz if there exists a constant $C>0$ such that for all $\lam_1$ and $\lam_2$ the frequency response is such that
\begin{equation}\label{eqn:lipschitzFilters}
   \big| h(\lam_2) - h(\lam_1) \big| 
       \leq C \big| \lam_2 - \lam_1 \big|.
\end{equation} \end{definition}

%
\noindent As its name suggests, a filter is Lipschitz if its frequency response is Lipschitz. This means a Lipschitz filter is one whose frequency response does not change faster than linear. For filters that are Lipschitz, the following stability result relative to perturbations that are close to permutations holds.

%
\begin{theorem} \label{thm:filterStabilityAbsolute}
Let $\bbS = \bbV \bbLambda \bbV^{\Hr}$ and $\hbS$ be graph shift operators. Let $\bbE = \bbU \bbM \bbU^{\Hr}$ be the absolute perturbation modulo permutation between $\bbS$ and $\hbS$ [cf. \eqref{eqn_def_operator_distance_error_matrix}] and assume their operator distance modulo permutation (cf. Def.~\ref{def_operator_distance}) satisfies
\begin{equation} \label{eqn:hypothesisEabsolute}
   \| \bbS - \hbS \|_\ccalP  = \| \bbE\| \leq \eps.
\end{equation}
For a Lipschitz filter (cf. Def.~\ref{def:lipschitzFilters}) with Lipschitz constant $C$ the operator distance modulo permutation between filters $\bbH(\bbS)$ and $\bbH(\hbS)$ satisfies
\begin{equation} \label{eqn:boundFilterAbsolute}
    \| \bbH(\bbS) - \bbH(\hbS) \|_\ccalP 
        \ \leq \ C
            \left(1 + \delta\sqrt{N} \right)\eps
            + \ccalO(\eps^{2})
\end{equation}
with $\delta := (\| \bbU - \bbV \|_{2}+1)^{2} -1$ standing for the eigenvector misalignment between  shift operator $\bbS$ and error matrix $\bbE$.
\end{theorem}

%
\begin{proof} See appendix \ref{sec_apx_B}. \end{proof}

%
To a first order approximation, Thm.~\ref{thm:filterStabilityAbsolute} shows that graph filters are Lipschitz stable with respect to absolute graph perturbations [cf. \eqref{eqn_def_operator_distance_error_matrix}-\eqref{eqn_def_operator_distance_error}]. The stability constant is given by $C(1 + \delta\sqrt{N})$ which implies that (i) the bound holds uniformly for all graphs with $N$ nodes, (ii) it is affected by a term, the Lipschitz constant $C$, that is controllable through filter design, and (iii) it is further affected by a term, the eigenvector misalignment $(1+\delta\sqrt{N})$, that depends on the structure of the perturbations that are expected in a particular problem but that cannot be affected by judicious filter choice. We note that the Lipschitz stability established by Thm.~\ref{thm:filterStabilityAbsolute} is with respect to the changes in the underlying graph support $\bbS$, and not with respect to changes in the input $\bbx$.

Although Thm.~\ref{thm:filterStabilityAbsolute} shows filter stability with respect to graph perturbations, the stability claim may be misleading given that the perturbation's norm is not tied to the norm of the graph shift. To do so we replace \eqref{eqn:hypothesisEabsolute} with the hypothesis $\| \bbS - \hbS \|_\ccalP = \| \bbE\| \leq \eps \|\bbS\|$, under which \eqref{eqn:boundFilterAbsolute} becomes
\begin{equation} \label{eqn:boundFilterAbsolute_dependent_on_S}
    \| \bbH(\bbS) - \bbH(\hbS) \|_\ccalP 
        \ \leq \ C
            \left(1 + \delta\sqrt{N} \right) \|\bbS\| \eps 
            + \ccalO(\eps^{2}) .
\end{equation}
The bound in \eqref{eqn:boundFilterAbsolute_dependent_on_S} is the result of a {\it relative} perturbation model. It is still a stability result but, in contrast to \eqref{eqn:boundFilterAbsolute}, not one that is uniform for all graphs with a given number of nodes. Making $\|\bbS\|$ arbitrarily large, makes the constant $(1 + \delta\sqrt{N}) \|\bbS\|$ arbitrarily large. One could think that it is reasonable to allow for larger filter perturbations when graphs have larger norm but this is not true. Filter perturbations determine feature perturbations whose magnitude need not be related to the graph's norm. On closer inspection the problem with making $\| \bbE\| \leq \eps \|\bbS\|$ is that the norms of $\bbE$ and $\bbS$ are global properties of the error and the graph. In particular, this may imply that parts of the graph with small weights have large relative modifications because some other parts of the graph have large weights. This observation prompts the relative perturbation model which effectively ties the local properties of the shift operator and error matrices, as discussed next.

%
\subsection{Effect of Relative Graph Perturbations on Graph Filters}\label{sec_graph_perturbations_relative}

We tie the perturbation of edges to their magnitudes by considering a relative perturbation model such that the relationship between a graph $\bbS$ and its perturbed version $\hbS$ is given by the error matrices $\bbE$ in the following set.

%
\begin{definition}[Relative Perturbation Modulo Permutation] \label{def_relative_perturbation} 
Given shift operators $\bbS$ and $\hbS$ we define the set of relative perturbation matrices modulo permutation as 
\begin{alignat}{2} \label{eqn_def_relative_perturbation}
\ccalE(\bbS, \hbS) 
= \Big\{ \bbE : 
\bbP^{\Tr}\hbS \bbP =  \bbS  + \big(\bbE \bbS + \bbS \bbE\big), 
\bbP \in \ccalP \Big\}. 
\end{alignat}
\end{definition}
    
The set of relative perturbation matrices modulo permutation considers all the matrices $\bbE$ that allow us to write different permutations of $\hbS$ through relative perturbations of $\bbS$ given by the model in \eqref{eqn_def_relative_perturbation}. The norm $\|\bbE\|$ of a given error matrix $\bbE$ associated to a given permutation $\bbP$ is a measure of relative dissimilarity between $\bbP^{\Tr}\hbS \bbP$ and $\bbS$. To measure dissimilarity between $\bbS$ and $\hbS$  modulo permutation we evaluate the error norm at the permutation that affords the smallest error norm
\begin{alignat}{2} \label{eqn_relative_perturbation_distance}
   d(\bbS, \hbS) 
      = & \min_{\bbP \in \ccalP} \ &&  \| \bbE \|   \nonumber \\
        & \st  && \bbP^{\Tr}\hbS \bbP =  \bbS  + \big(\bbE \bbS + \bbS \bbE\big) .
\end{alignat}
The dissimilarity $d(\bbS, \hbS)$ measures how close $\bbS$ and $\hbS$ are to being permutations of each other, as determined by the multiplicative factor $\bbE$. Such a model ties changes in the edge weights of the graph to its local structure. To see this, note that the difference between the edge weight $s_{ij}$ of the original graph $\bbS$ and the corresponding edge $[\bbP_0^{\Tr}\hbS\bbP_0]_{ij}$ of the perturbed graph $\hbS$ is given by the corresponding entry $[\bbE\bbS+\bbS\bbE]_{ij}$ of the perturbation factor $\bbE\bbS+\bbS\bbE$. It is ready to see that this quantity is proportional to the sum of the degrees of nodes $i$ and $j$ scaled by the entries of $\bbE$. As the norm $\|\bbE\|$ grows, the entries of the graphs $\bbS$ and $\bbP_0^{\Tr}\hbS\bbP_0$ become more dissimilar. But parts of the graph that are characterized by weaker connectivity change by amounts that are proportionally smaller to the changes that are observed in parts of the graph characterized by stronger links. This is in contrast to absolute perturbations where edge weights change by the same amount irrespective of the local topology of the graph.

To study the effect of relative perturbations we can rely on the consideration of Lipschitz filters (cf. Def.~\ref{def:lipschitzFilters}) but more illuminating results are possible with the use of integral Lipschitz filters, which satisfy a condition on the rate of change of their frequency responses that we introduce next.

%
\begin{definition}[Integral Lipschitz Filter]\label{def:integralLipschitzFilters}
    Given a filter $\bbh = \{h_{k}\}_{k=0}^\infty$ its frequency response $h(\lam)$ is given by \eqref{eqn:frequency_response} and satisfies $|h(\lam)|\leq 1$. We say the filter is integral Lipschitz if there exists a constant $C>0$ such that for all $\lam_1$ and $\lam_2$,
    \begin{equation}\label{eqn:integralLipschitzFilters}
    | h(\lam_2) - h(\lam_1) |
    \ \leq\  C\ \frac{| \lam_2 - \lam_1 |}{  |\lam_1 +  \lam_2| \,/\, 2  }\, .
    \end{equation} \end{definition}

%
The condition in \eqref{eqn:integralLipschitzFilters} can be read as requiring the filter's frequency response to be Lipschitz in any interval $(\lam_1,\lam_2)$ with a Lipschitz constant that is inversely proportional to the interval's midpoint $(|\lam_1 + \lam_2|)/2$. To see this better, observe that \eqref{eqn:integralLipschitzFilters} restricts the frequency response's derivative to satisfy,
\begin{equation}\label{eqn_integral_Lipschitz_derivative}
\big|\lambda h'(\lam) \big| \leq C .
\end{equation}
Thus, filters that are integral Lipschitz must have frequency responses that have to be flat for large $\lam$ but can vary very rapidly around $\lam=0$. This is, not coincidentally, a condition reminiscent of the scale invariance of wavelet transforms \cite[Ch.~7]{Daubechies92-Wavelets}. The condition $|h(\lam)|\leq 1$ is not necessary but it eases interpretations by preventing the filter from amplifying energy. 

    
For relative perturbation models and integral Lipschitz filters the following stability result holds.

%
\begin{theorem} \label{thm:filterStabilityRelative}
Let $\bbS = \bbV \bbLambda \bbV^{\Hr}$ and $\hbS$ be graph shift operators. Let $\bbE = \bbU \bbM \bbU^{\Hr} \in \ccalE(\bbS,\hbS)$ be a relative perturbation matrix (cf. Def.~\ref{def_relative_perturbation}) whose norm is such that [cf. \eqref{eqn_relative_perturbation_distance}]
\begin{equation} \label{eqn:hypothesisErelative}
   d(\bbS, \hbS) \leq \| \bbE \| \leq \eps.
\end{equation}
For an integral Lipschitz filter (cf. Def.~\ref{def:integralLipschitzFilters}) with integral Lipschitz constant $C$ the operator distance modulo permutation between filters $\bbH(\bbS)$ and $\bbH(\hbS)$ satisfies
\begin{equation} \label{eqn:boundFilterRelative}
   \| \bbH(\bbS) - \bbH(\hbS) \|_\ccalP  
       \ \leq \ 2 C \left(1 + \delta\sqrt{N} \right)\eps + \ccalO(\eps^{2})
\end{equation}
with $\delta := (\| \bbU - \bbV \|_{2}+1)^{2} -1$ standing for the eigenvector misalignment between  shift operator $\bbS$ and error matrix $\bbE$
\end{theorem}

%
\begin{proof} See appendix \ref{sec_apx_C}. \end{proof}

%
Thm.~\ref{thm:filterStabilityRelative} establishes stability with respect to relative perturbations of the form introduced in Def.~\ref{def_relative_perturbation}. If a matrix $\bbE$ exists that makes $\bbS$ and $\hbS$ close to permutations of each other in terms of this relative perturbation, the filters are stable with respect to the norm of the perturbation with stability constant $2C (1 + \delta\sqrt{N} )$. The constant has the same shape as the one in Thm.~\ref{thm:filterStabilityAbsolute}, and while this is a coincidence, similar observations hold. Namely, the bound is uniform for all graphs with the same number of nodes, the stability is affected by the integral Lipschitz constant $C$ which depends on the filter, and it is also affected by the eigenvector misalignment $(1+\delta\sqrt{N})$, which depends on the structure of the perturbation. The important difference between Thms.~\ref{thm:filterStabilityAbsolute}~and~\ref{thm:filterStabilityRelative} is that the meaning of $C$ is different since the class of filters that are admissible for stability with respect to relative perturbations is that of integral Lipschitz filters -- whereas Lipschitz filters are required for stability with respect to absolute perturbations.

Before elaborating on the implications of allowing for integral Lipschitz filters we consider a variation of Thm.~\ref{thm:filterStabilityRelative} in which we impose a structural constraint on the perturbation.

%
\begin{theorem} \label{thm_filter_stability_structural_constraint} 
With the same hypotheses and definitions of Thm.~\ref{thm:filterStabilityRelative} assume that there exists a matrix $\bbE \in \ccalE(\bbS,\hbS)$ that satisfies \eqref{eqn:hypothesisErelative} and, furthermore, is such that
\begin{equation} \label{eqn:structuralConstraint}
   \min\left [ \left\| \frac{\bbE}{\|\bbE\|} - \bbI \right\| \, , \
               \left\| \frac{\bbE}{\|\bbE\|} + \bbI \right\|  \right] 
       \leq \varepsilon .
\end{equation}
Then, the operator distance modulo permutation between filters $\bbH(\bbS)$ and $\bbH(\hbS)$ satisfies
\begin{equation}\label{eqn_thm_filter_stability_structural_constraint} 
   \| \bbH(\bbS) - \bbH(\hbS) \|_{\ccalP} 
        \leq  2C \varepsilon + \ccalO(\varepsilon^{2}) .
\end{equation} \end{theorem}

%
\begin{proof} See appendix \ref{sec_apx_C}. \end{proof}

%
The structural constraint in \eqref{eqn:structuralConstraint} requires the error matrix $\bbE$ to be a scaled identity to within a first order approximation. With this restriction on the set of admissible perturbations we can bound the eigenvector misalignment between $\bbS$ and $\bbE$ and remove the dependency that the bound in Thm.~\ref{thm:filterStabilityRelative} has on the number of nodes $N$. The bound in \eqref{eqn_thm_filter_stability_structural_constraint} holds uniformly for all graphs independently of their number of nodes.

The integral Lipschitz filters in Thms.~\ref{thm:filterStabilityRelative}~and~\ref{thm_filter_stability_structural_constraint} are of interest because they can be made finely discriminative at low-eigenvalue frequencies without affecting stability. Indeed, to control stability in \eqref{eqn:boundFilterRelative} we need to limit the value of the integral Lipschitz constant $C$. This requires filters that change more slowly. In particular, for large $\lam$ the filters must be constant and cannot discriminate nearby spectral features. But at values of $\lam\approx 0$ the filters can change rapidly and can therefore be designed to discriminate (arbitrarily) close spectral features. To the extent that relative perturbations are admissible -- which, as already discussed, are arguably more sensible than absolute ones -- Thm.~\ref{thm:filterStabilityRelative} shows two fundamental properties of linear graph filters: (i) They {\it cannot} be stable and discriminative of spectral features associated with large $\lam$. (ii) They {\it can} be stable and discriminative of spectral features associated with $\lam\approx0$.

Thus, if we are interested in discriminating features associated with $\lam\approx0$, linear graph filters are sufficient. However, if we are interested in discriminating features associated with large $\lam$, linear graph filters will fail because of their sensitivity to graph perturbations. We will see in the following section that this is an issue we can resolve with the introduction of pointwise nonlinearities to produce graph neural networks.

%
\begin{remark}[Integral Lipschitz filters] \normalfont
    We note that filter banks abiding to the integral Lipschitz condition exist in the literature; see, for example, graph wavelets \cite{Coifman06-DiffusionWavelets, Shuman15-Wavelets, Hammond11-Wavelets}. However, filters in a GNN are trained from data and it is therefore not necessarily guaranteed that they will satisfy this condition. To address this from a practical standpoint, in Sec.~\ref{sec:sims} we add the integral Lipschitz condition as a penalty during training \eqref{eqn_integral_Lipschitz_derivative}. This results in filters with controllable constant $C$ that illustrate different degrees of stability. We further remark that, regardless of enforcing integral Lipschitz conditions on the filters learned in a GNN, the insights stemming from the discussion ensuing in Sec.~\ref{sec_linear_filters_stability} still hold. Namely, that GNNs are simultaneously stable and discriminative, a feat that cannot be achieved by linear filter banks by themselves.
\end{remark}

%
\begin{remark}[Eigenvector misalignment] \normalfont
    The bound \eqref{eqn:boundFilterRelative} in Theorem~\ref{thm:filterStabilityRelative} depends on the eigenvector misalignment constant $\delta$. This, in turn, depends on the specific relative perturbation $\bbE \in \ccalE$ [cf. \eqref{eqn_def_relative_perturbation}] and computing it requires an eigendecomposition. However, there are two important observations to be made. First, that while Theorem~\ref{thm:filterStabilityRelative} holds for all values of $\bbE$, there are specific families of perturbations for which the eigenvector misalignment constant can be known (as is the case with Euclidean perturbations). Second, that regardless of the specific $\bbE$, it always holds that $\delta \leq 8$, following from the fact that $\|\bbU\| \leq 1$ and $\|\bbV\| \leq 1$ because they are eigenvector matrices. However, this accentuates the dependence of the Lipschitz constant with the size of the graph $N$.
\end{remark}

%% file: plots_stability/frequency_response_instantiation.tex

\def \thisplotscale {2.9}
\def \unit {\thisplotscale cm}

\def \frequencyresponse 
     {   0.8*exp(-(1*(x-1.2))^2) 
       + 0.7*exp(-(0.7*(x-4))^2) 
       + 0.8*exp(-(1.1*(x-6))^2) 
       + 0.1}

\begin{tikzpicture}[x = 1*\unit, y=1*\unit]
\begin{axis}[scale only axis,
             width  = 2.4*\unit,
             height = 1*\unit,
             xmin = -0.5, xmax=7.5,
             xtick = {0, 0.45, 2.63, 3.37, 6.72, 7.20},
             xticklabels = {\blue{$\lam_1\phantom{\hlam}$}, 
                            \red{$\hlam_1$}, 
                            \blue{$\lam_i\phantom{\hlam}$}, 
                            \red{$\hlam_i$},
                            \blue{$\lam_{N}\phantom{\hlam}$},
                            \red{$\hlam_N$}},
             ymin = -0, ymax = 1.15,
             ytick = {-1},
             typeset ticklabels with strut,
             enlarge x limits=false]

\addplot+[samples at = {0.00, 0.91, 1.57, 
                        2.63, 3.97, 4.51, 
                        5.60, 6.72}, 
          color = blue!60, 
          ycomb, 
          mark=otimes*, 
          mark options={blue!60}]
         {\frequencyresponse};

\addplot+[samples at = {0.45, 1.25, 2.31, 
                        3.37, 4.23, 5.15, 
                        6.20, 7.20}, 
          color = red!60, 
          ycomb, 
          mark=oplus*, 
          mark options={red!60}]
         {\frequencyresponse};

\addplot[ domain=-0.5:7.5, 
          samples = 80, 
          color = black,
          line width = 1.2]
         {\frequencyresponse};

\end{axis}
\end{tikzpicture}


%% file: GNNstability.tex

%
\section{Stability Properties of Graph Neural Networks} \label{sec:graphNeuralNetworks}

Graph neural networks are a cascade of layers, each of which applies a bank of graph filters, followed by a pointwise nonlinearity \cite{Bruna14-DeepSpectralNetworks, Defferrard17-CNNGraphs, Kipf17-ClassifGCN, Gama18-Architectures}. To increase the representation power of GNNs, we consider that at layer $\ell$ there are $F_{\ell}$ graph signals $\bbx_{\ell}^{f} \in \reals^{N}$, $f=1,\ldots,F_{\ell}$ instead of a single one, as was the case in the previous section. Each graph signal $\bbx_{\ell}^{f}$ is called a \emph{feature}. The input to layer $\ell$ is then the $F_{\ell-1}$ signals $\bbx_{\ell-1}^{g}$ that were the output of layer $\ell-1$, $g=1,\ldots,F_{\ell-1}$. To compute the $F_{\ell}$ features $\bbx_{\ell}^{f}$ we expect at the output of layer $\ell$, we first process the $F_{\ell-1}$ input signals with a bank of $F_{\ell-1}F_{\ell}$ graph filters denoted by $\bbH_{\ell}^{fg}(\bbS)$ [cf. \eqref{eqn:graphFilter}], defined by coefficients $\bbh_{\ell}^{fg}$
\begin{equation} \label{eqn_gnn_filter}
    \bbz_{\ell}^{fg}  =  \sum_{k=0}^{\infty} h_{\ell k}^{fg}\, \bbS^{k} \bbx_{\ell-1}^{g} = \bbH_{\ell}^{fg}(\bbS) \bbx_{\ell-1}^{g}
\end{equation}
where we obtain the intermediate features $\bbz_{\ell}^{fg}$ for $f=1,\ldots,F_{\ell}$ and $g=1,\ldots,F_{\ell-1}$. All of these intermediate features $\bbz_{\ell}^{fg}$ for a given index $f$ are summed together to linearly yield $F_{\ell}$ features, and passed through a pointwise nonlinear function $\sigma:\reals\to\reals$ to produce the output feature
\begin{equation} \label{eqn_gnn_nonlinearity}
    \bbx_{\ell}^{f}  \ = \ \sigma \bigg[\, \sum_{g} \bbz_{\ell}^{fg}\, \bigg]
\end{equation}
for $f=1,\ldots,F_{\ell}$. We note that, in an abuse of notation, the application of $\sigma$ to the vector $\sum_{g}\bbz_{\ell}^{fg}$ implies the entrywise application of the nonlinearity, i.e. the same $\sigma$ is applied independently to every feature at every node. The input to the GNN is the $0$th layer signal $\bbx = \bbx_0^1$ and the output of the GNN is the $L$th layer feature $\bbx_L^1$. 

We note that \eqref{eqn_gnn_filter}-\eqref{eqn_gnn_nonlinearity} can be compactly written as
\begin{equation}
    \bbX_{\ell} = \sigma \Big[ \sum_{k = 0}^{\infty} \bbS^{k} \bbX_{\ell-1} \bbH_{\ell k}\Big]
\end{equation}
for $\bbX_{\ell} \in \reals^{N \times F_{\ell}}$ the matrix whose columns are the graph signal features $\bbx_{\ell}^{f}$ and where $\bbH_{\ell k} \in \reals^{F_{\ell-1} \times F_{\ell}}$ is the matrix collecting the $k$th coefficient of all filters in the bank, $[\bbH_{\ell k}]_{gf} = h_{\ell k}^{fg}$. In what follows, however, we choose the description in \eqref{eqn_gnn_filter}-\eqref{eqn_gnn_nonlinearity} which emphasizes the role of each graph filter, allowing us to better focus on the interaction between filters and nonlinearities. For ease of exposition, we assume that at each layer each feature is processed by $F$ filters. This means that the first and last layer contain $F$ filters whereas the remaining intermediate layers contain $F^2$ filters. In any case, the results derived here extend to intermediate layers with varying number of features in a straightforward manner (see proofs in the Appendix).

We emphasize that the nonlinear operation in \eqref{eqn_gnn_nonlinearity} is applied to each entry of $\bbz_{\ell}^{f}$ individually. We further assume that the nonlinearity is normalized Lipschitz so that for all $a,b\in\reals$,
\begin{equation} \label{eqn_Lipschitz_nonlinearity}
	 | \sigma(b) - \sigma(a) | \leq | b - a | .
\end{equation}
Asides from the input $\bbx$, the GNN's output depends on the filters $\bbh_{\ell}^{fg}$ and the graph $\bbS$. We interpret a GNN as a transform defined by the filter coefficients that we can apply on {\it any graph} to any signal defined on the graph. Define then the map
\begin{equation} \label{eqn_gnn_operator}
   \Phi(\bbS,\bbx) \ = \ \bbx_L^1
\end{equation}
to represent the outcome of applying \eqref{eqn_gnn_filter}-\eqref{eqn_gnn_nonlinearity} on graph $\bbS$ to input signal $\bbx = \bbx_0^1$. Our goal is to study the stability of the operator $\Phi(\bbS,\cdot)$ with respect to perturbations of the graph $\bbS$.

%
\begin{remark}[Graph neural networks] \normalfont
    We consider GNNs defined by equations \eqref{eqn_gnn_filter}-\eqref{eqn_gnn_nonlinearity}. The literature includes several different proposed architectures such as ChebNets \cite{Defferrard17-CNNGraphs}, GCNs \cite{Kipf17-ClassifGCN} or Selection GNNs \cite{Gama18-Architectures}. The model in equations \eqref{eqn_gnn_filter}-\eqref{eqn_gnn_nonlinearity} is the one that appears in \cite{Gama18-Architectures} and it could be therefore interpreted as a particular choice. However, it is known that all of the architectures in 
    \cite{Defferrard17-CNNGraphs}-\cite{Gama18-Architectures} can be equivalently described by the GNN model in equations \eqref{eqn_gnn_filter}-\eqref{eqn_gnn_nonlinearity} \cite{Gama19-GraphConv, Gama20-SPM}. Therefore, the results that we derive can be applied to all of these architectures, which are formally known as graph \emph{convolutional} neural networks. For a comprehensive framework of convolutional as well as non-convolutional graph neural networks, as well as a proof of equivalence between architectures, see \cite{Isufi20-EdgeNets}.
\end{remark} 

%
\subsection{Permutation Equivariance of GNNs}\label{sec_gnn_stability_perturbations}

The superior performance of graph neural networks (GNNs) can be explained by the use of filter banks and nonlinearities to successfully process high-eigenvalue frequencies in a stable manner. An arbitrary GNN $\Phi(\bbS, \cdot)$ with $L$ layers, over a graph representation $\bbS$, is defined by \eqref{eqn_gnn_filter}-\eqref{eqn_gnn_nonlinearity} for $\ell=1,\ldots,L$. GNNs retain the two fundamental properties of linear filters. Namely, permutation equivariance and stability.

%
\begin{proposition}[GNN permutation equivariance] \label{prop:GNNPermutationEquivariance}
    Consider graph shifts $\bbS$ and $\hbS = \bbP^{\Tr} \bbS \bbP$ for some permutation matrix $\bbP \in \ccalP$ [cf. \eqref{eqn:permutationSet}]. Given a bank of filters $\{\bbh_{\ell}^{fg}\}$ for each layer $\ell=1,\ldots,L$ and a pointwise nonlinearity $\sigma$, define a GNN $\Phi$ [cf. \eqref{eqn_gnn_filter}-\eqref{eqn_gnn_nonlinearity}]. Then, for any pair of corresponding graph signals $\bbx$ and $\hbx = \bbP^{\Tr} \bbx$ used as input to the GNN it holds that
    \begin{equation}
    \Phi(\hbS, \hbx) = \bbP^{\Tr} \Phi(\bbS, \bbx).
    \end{equation}
\end{proposition}
%
%
\begin{proof}
    See Appendix~\ref{sec_apx_D}.
\end{proof}
%

Proposition~\ref{prop:GNNPermutationEquivariance} states that GNNs retain the permutation equivariance inherited from graph filters [cf. Prop~\ref{prop:filterPermutationEquivariance}], so that graph signal processing with GNNs is independent of node relabelings. To study stability of the GNN operator in \eqref{eqn_gnn_operator}, we therefore want to consider measures of proximity that are impervious to permutations. To that end we define an operator distance modulo permutation as follows.

%
\begin{definition}[Operator Distance Modulo Permutation] \label{def_nonlinear_operator_distance} 
Given operators $\Psi:\reals^N\to\reals^N$ and $\hat{\Psi}:\reals^N\to\reals^N$ we define their operator distance modulo permutation as
\begin{align} \label{eqn_def_nonlinear_operator_distance}
   \big\| \Psi - \hat\Psi \big\|_{\ccalP} 
      \ = \   \min_{\bbP\in\ccalP} \,
                 \max_{ \bbx : \| \bbx \| = 1 } \,
                     \big\|\bbP^{\Tr}\Psi(\bbx)  - \hat\Psi(\bbP^{\Tr} \bbx)\big\|
\end{align} 
where $\ccalP$ is the set of $N\times N$ permutation matrices (cf. \eqref{eqn:permutationSet}) and where $\| \cdot \|$ stands for the $\ell_{2}$-norm.
\end{definition}

%
The operator distance in \eqref{eqn_def_nonlinear_operator_distance} compares operators $\Psi$ and $\hat{\Psi}$ when the same permutations are applied at their respective inputs and output, and it is a generalization of Def.~\ref{def_operator_distance} to (nonlinear) operators. GNNs are insensitive to permutations, as shown by Prop.~\ref{prop:GNNPermutationEquivariance}, and thus we have $\big\| \Phi - \hat\Phi \big\|_{\ccalP} = 0$. The distance in \eqref{eqn_def_nonlinear_operator_distance} is thus a measure of how far from a permutation the operators are.

%
\subsection{Stability of GNNs to Perturbations of the Graph}

The stability of GNNs is inherited from that of the graph filters that conform the filter bank used in \eqref{eqn_gnn_filter}. The hyperparameters of the GNN further impact the stability.

%
\begin{theorem}[GNN Stability] \label{thm:GNNStability}
    Let $\bbS$ and $\hbS$ be GSOs related by perturbation matrix $\bbE$ [cf. \eqref{eqn_def_operator_distance_error_matrix} or \eqref{eqn_relative_perturbation_distance}] such that $\| \bbE \| \leq \varepsilon$. Given a bank of filters $\{\bbh_{\ell}^{fg}\}$ such that $|h_{\ell}^{fg}(\lambda)| \leq 1$ [cf. \eqref{eqn:frequency_response}] and a pointwise nonlinearity $\sigma$ that is Lipschitz continuous \eqref{eqn_Lipschitz_nonlinearity}, define GNNs $\Phi(\bbS, \cdot)$ and $\Phi(\hbS, \cdot)$ [cf. \eqref{eqn_gnn_filter}-\eqref{eqn_gnn_nonlinearity}]. If the corresponding filter banks satisfy $\| \bbH_{\ell}^{fg}(\bbS) - \bbH_{\ell}^{fg}(\hbS) \|_{\ccalP} \leq \Delta \varepsilon$, then it holds that
    \begin{equation} \label{eqn:genericStabilityConstant}
    \| \Phi(\bbS, \cdot) - \Phi(\hbS, \cdot) \|_{\ccalP} \leq \Delta LF^{L-1} \varepsilon + \ccalO(\varepsilon^{2})
    \end{equation}
    for a GNN with a single input feature, a single output feature and $F$ features in each hidden layer.
\end{theorem}
%
%
\begin{proof}
    See Appendix \ref{sec_apx_E}.
\end{proof}
%

Thm.~\ref{thm:GNNStability} establishes how the stability of the filters $\Delta$ is affected by the hyperparameters of the GNN architecture. More specifically, we see that the stability gets degraded linearly with the number of layers $L$, and exponentially with the number of features $F$ (with an exponent controlled by $L$). In essence, the deeper a GNN is, the less stable it is. We note that this causes the bound to be quite loose, as is also evidenced in Sec.~\ref{sec:sims}. However, we see that the result is still linear in the size of the perturbation $\varepsilon$ and in the stability constant $\Delta$ of the filters. This stability constant depends on the perturbation model under consideration (either absolute --Sec.~\ref{sec_graph_perturbations}-- or relative --Sec.~\ref{sec_graph_perturbations_relative}--) and on the Lipschitz condition on the graph filters (either Lipschitz --Def.~\ref{def:lipschitzFilters}-- or integral Lipschitz --Def.~\ref{def:integralLipschitzFilters}--), as determined next.

%
\begin{proposition}\label{prop:Delta}
Under the conditions of Thm.~\ref{thm:GNNStability}, with $\bbS = \bbV \bbLambda \bbV^{\Hr}$, consider the following models.
\begin{enumerate}[(i)]
    \vspace{.2cm}
    \item If matrix $\bbE = \bbU \bbM \bbU^{\Hr}$ models absolute perturbations [cf. \eqref{eqn_def_operator_distance_error_matrix}], and the filters are Lipschitz (Def.~\ref{def:lipschitzFilters}) we have
    \begin{equation} \label{eqn:DeltaAbs}
    \bbDelta = C (1+\delta \sqrt{N})
    \end{equation}
    with $\delta = (\| \bbU - \bbV \| + 1)^{2}-1$.
    \vspace{.4cm}
    \item If matrix $\bbE = \bbU \bbM \bbU^{\Hr}$ models relative perturbations (Def.~\ref{def_relative_perturbation}), and the filters are integral Lipschitz (Def.~\ref{def:integralLipschitzFilters}),
    \begin{equation} \label{eqn:DeltaRel}
    \bbDelta = 2 C (1+\delta \sqrt{N})
    \end{equation}
    holds with $\delta = (\| \bbU - \bbV \| + 1)^{2}-1$.
    \vspace{.4cm}
    \item If matrix $\bbE$ models relative perturbations (Def.~\ref{def_relative_perturbation}) and satisfies \eqref{eqn:structuralConstraint}, and the filters are integral Lipschitz (Def.~\ref{def:integralLipschitzFilters}),
    \begin{equation} \label{eqn:DeltaStructural}
    \bbDelta = 2C.
    \end{equation}
\end{enumerate}
\end{proposition}
%
%
\begin{proof}
    Follows directly from Thm.~\ref{thm:GNNStability} in combination with Thms.~\ref{thm:filterStabilityAbsolute},~\ref{thm:filterStabilityRelative}~and~\ref{thm_filter_stability_structural_constraint}. These theorems establish the conditions and the corresponding values of $\Delta$.
\end{proof}
%

The results of Thm.~\ref{thm:GNNStability} in combination with Prop.~\ref{prop:Delta} show that: (i) the use of Lipschitz filters lead to stable GNNs under absolute perturbations, and (ii) the use of integral Lipschitz filters lead to stable GNNs under relative perturbations. In all cases, Thm.~\ref{thm:GNNStability} establishes that GNNs $\bbPhi(\bbS, \bbx)$ are Lipschitz stable with respect to the changes in the underlying graph support $\bbS$, and not with respect to changes in the input $\bbx$.

The stability constant in \eqref{eqn:genericStabilityConstant} in combination with the value of $\Delta$ in \eqref{eqn:DeltaAbs} or \eqref{eqn:DeltaRel} consists of the product of three terms. One given by the filter's (integral) Lipschitz constant, $C$, one given by the number of filters and layers in the GNN, $L F^{L-1}$, and one containing the eigenvector misalignment constant $(1+\delta\sqrt{N})$. The one related to the GNN architecture is just a consequence of perturbations propagating across different filters. The other two terms represent different fundamental facets of GNNs. The (integral) Lipschitz constant $C$ is a property of the filters which is up for choice during filter design, or, perhaps more likely, expected as an outcome of the training process. The term $(1+\delta\sqrt{N})$ is a property of the family of perturbations $\bbE$ that are admissible which is an inherent property of the type of perturbations we expect to see in a specific problem. It is important to remark that we can affect $C$ by designing or learning proper filters but we cannot affect $\delta$. The latter is not a property of the filter, but a property of the perturbation $\bbE$. The important point is that, regardless of $\delta$, judicious choice of filter coefficients $\bbh$, affects the (integral) Lipschitz constant $C$ and allows control of the stability of the graph filters that define a GNN.

The value of $\varepsilon$ in model (i) of Prop.~\ref{prop:Delta} represents the absolute perturbation distance [cf. \eqref{eqn_def_operator_distance_error_matrix}] between shifts $\bbS$ and $\hbS$ and as such, is independent of the actual particularities of the graph under study (a fixed value of $\varepsilon$ would mean a different perturbation level for graphs that have very different edge weights). Likewise, the value of $C$ is given by the Lipschitz constant of the filters. The higher the value of $C$, the more selective the filters can be (the more narrow they can be), but the more unstable the GNNs become. Finally, the value of $\delta$ accounts for the eigenvector misalignment between the absolute error matrix $\bbE$ and the shift $\bbS$, which indicates the impact on the spectrum basis by the perturbation, and affects the stability bound by a value dependent on the number of nodes (the larger the graph, the more a change in the spectrum basis affects stability).

%
\begin{figure}[t]
    \centering
    \input{plots_stability/graph_dilation_unstable.tex}\medskip
    \input{plots_stability/graph_dilation_stable.tex}    
    \caption{Stability of graph filters. We observe that, for small values of $\lambda$, the difference between $\lambda_{i}$ (in blue) and $\hlam_{i}$ (in red) is small, whereas for large $\lambda$ this becomes much larger. (top) When using a Lipschitz filter [cf. \eqref{eqn:lipschitzFilters}], we observe that for low frequencies, the response of the filter is very similar when instantiated on either $\lambda_{i}$ or $\hlam_{i}$; however, for large frequencies, the difference becomes much larger, and thus a small change in the eigenvalues, leads to a big change of the filter response. (bottom) In the case of integral Lipschitz filters [cf. \eqref{eqn:integralLipschitzFilters}], the effect on high frequencies is mitigated, by forcing the filter to be nearly constant at these frequencies, so that, when evaluated at eigenvalues that are far away, the filter response is still almost the same, guaranteeing stability.}
    \label{fig_graph_dilation}
\end{figure}
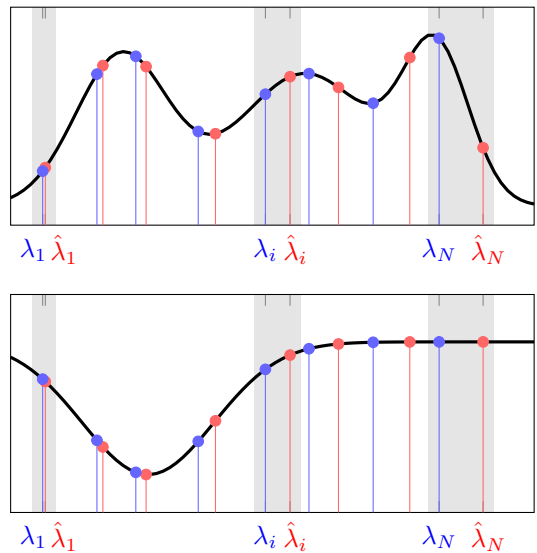

With respect to model (ii) of Prop.~\ref{prop:Delta} we observe that now $\varepsilon$ represents the relative distance between $\bbS$ and its perturbation $\hbS$ [cf.~\eqref{eqn_relative_perturbation_distance}], meaning that a fixed $\varepsilon$ represents the same level of perturbation for any possible reweighing of the difference $\alpha (\bbS - \hbS)$, $\alpha \in \reals$. The value of $C$, in this case, represents the integral Lipschitz constant of the filters (cf.~Def.~\ref{def:integralLipschitzFilters}). Integral Lipschitz filters, however, can be made arbitrarily selective near $\lambda \approx 0$, irrespective of the value of $C$, allowing for perfect discrimination of features around it, without affecting the overall stability. In integral Lipschitz filters, the value of $C$ determines the smallest eigenvalue for which the filter response becomes (approximately) flat, and hence loses discriminative power. A high value of $C$ would allow for greater selectivity in higher-eigenvalue frequencies, but at the expense of stability. With respect to $\delta$, the same analysis as for model (i) holds, except that in this case, the eigenvectors $\bbU$ correspond to the \emph{relative} error matrix $\bbE$. We also note that the presence of $\delta$ causes the bound to be quite loose for large values of $N$.

To overcome the degradation of the stability with the size of the graph, we propose model (iii) of Prop.~\ref{prop:Delta}. In this model, where $\varepsilon$ measures the relative perturbation distance and $C$ the integral Lipschitz constant, the family of admissible perturbations has been restricted to those that satisfy the structural constraint \eqref{eqn:structuralConstraint}. Admissible perturbations are now those that are either dilations or contractions of the edge weights of $\bbS$. Dilations and contractions can be different for different nodes but cannot be a mix of dilation and contraction in different parts of the graph. We remark that if the structural constraint is satisfied, then the stability can be controlled by determining the integral Lipschitz constant of the filters, for any graph. However, for some specific families of graphs, where we have information on how the eigenvectors change with a given perturbation size, we can improve on the result by relaxing the structural constraint. This is the case of \cite{Mallat12-Scattering}, where extraneous geometric information (Euclidean space) is leveraged to quantify the impact of the perturbation (diffeomorphism) on the spectrum basis.

%% file: plots_stability/graph_dilation_unstable.tex

\def \thisplotscale {2.9}
\def \unit {\thisplotscale cm}

\def \frequencyresponse 
     {   0.8*exp(-(1*(x-1.2))^2) 
       + 0.7*exp(-(0.7*(x-4))^2) 
       + 0.8*exp(-(1.4*(x-6))^2) 
       + 0.1}

\hspace{-2.9mm}
\begin{tikzpicture}[x = 1*\unit, y=1*\unit]

\def \factorx {2.4/8}
\def \deltax  {0.5*\factorx}
\def \shadeshift  {0.05}

\path [fill=black!20, opacity = 0.5] 
              (\deltax - 0.001*\factorx - \shadeshift, 0.00) rectangle 
              (\deltax + 0.030*\factorx + \shadeshift, 1.00);
\path [fill=black!20, opacity = 0.5] 
              (\deltax + 3.393*\factorx - \shadeshift, 0.00) rectangle 
              (\deltax + 3.770*\factorx + \shadeshift, 1.00);
\path [fill=black!20, opacity = 0.5] 
              (\deltax + 6.048*\factorx - \shadeshift, 0.00) rectangle 
              (\deltax + 6.720*\factorx + \shadeshift, 1.00);

\begin{axis}[scale only axis,
             width  = 2.4*\unit,
             height = 1*\unit,
             xmin = -0.5, xmax=7.5,
             xtick = {0.03, -0.01, 3.77, 3.393, 6.72, 6.048},
             xticklabels = {\red{$\qquad\hlam_1\phantom{\lam}$},
                            \blue{$\lam_1\ \ $}, 
                            \red{$\quad\hlam_i\phantom{\lam}$}, 
                            \blue{$\lam_i$},
                            \red{$\quad\hlam_{N}\phantom{\lam}$},
                            \blue{$\lam_N$}},
             ymin = -0, ymax = 1.15,
             ytick = {-1},
             typeset ticklabels with strut,
             enlarge x limits=false]

\addplot+[samples at = {0.03, 0.91, 1.57, 
                        2.63, 3.77, 4.51, 
                        5.60, 6.72}, 
          color = red!60, 
          ycomb, 
          mark=otimes*, 
          mark options={red!60}]
         {\frequencyresponse};

\addplot+[samples at = {-0.01, 0.819, 1.413, 
                        2.367, 3.393, 4.059, 
                        5.04, 6.048}, 
          color = blue!60, 
          ycomb, 
          mark=oplus*, 
          mark options={blue!60}]
         {\frequencyresponse};

\addplot[ domain=-0.5:7.5, 
          samples = 80, 
          color = black,
          line width = 1.2]
         {\frequencyresponse};

\end{axis}
\end{tikzpicture}


%% file: plots_stability/graph_dilation_stable.tex

\def \thisplotscale {2.9}
\def \unit {\thisplotscale cm}

\def \frequencyresponse 
     { 0.9 - 0.7*exp(-(0.7*(x-1.6))^2) }

\hspace{-2.9mm}
\begin{tikzpicture}[x = 1*\unit, y=1*\unit]

\def \factorx {2.4/8}
\def \deltax  {0.5*\factorx}
\def \shadeshift  {0.05}

\path [fill=black!20, opacity = 0.5] 
              (\deltax - 0.001*\factorx - \shadeshift, 0.00) rectangle 
              (\deltax + 0.030*\factorx + \shadeshift, 1.00);
\path [fill=black!20, opacity = 0.5] 
              (\deltax + 3.393*\factorx - \shadeshift, 0.00) rectangle 
              (\deltax + 3.770*\factorx + \shadeshift, 1.00);
\path [fill=black!20, opacity = 0.5] 
              (\deltax + 6.048*\factorx - \shadeshift, 0.00) rectangle 
              (\deltax + 6.720*\factorx + \shadeshift, 1.00);

\begin{axis}[scale only axis,
             width  = 2.4*\unit,
             height = 1*\unit,
             xmin = -0.5, xmax=7.5,
             xtick = {0.03, -0.01, 3.77, 3.393, 6.72, 6.048},
             xticklabels = {\red{$\qquad\hlam_1\phantom{\lam}$},
                            \blue{$\lam_1\ \ $}, 
                            \red{$\quad\hlam_i\phantom{\lam}$}, 
                            \blue{$\lam_i$},
                            \red{$\quad\hlam_{N}\phantom{\lam}$},
                            \blue{$\lam_N$}},
             ymin = -0, ymax = 1.15,
             ytick = {-1},
             typeset ticklabels with strut,
             enlarge x limits=false]

\addplot+[samples at = {0.03, 0.91, 1.57, 
                        2.63, 3.77, 4.51, 
                        5.60, 6.72}, 
          color = red!60, 
          ycomb, 
          mark=otimes*, 
          mark options={red!60}]
         {\frequencyresponse};

\addplot+[samples at = {-0.01, 0.819, 1.413, 
                        2.367, 3.393, 4.059, 
                        5.04, 6.048}, 
          color = blue!60, 
          ycomb, 
          mark=oplus*, 
          mark options={blue!60}]
         {\frequencyresponse};

\addplot[ domain=-0.5:7.5, 
          samples = 80, 
          color = black,
          line width = 1.2]
         {\frequencyresponse};

\end{axis}
\end{tikzpicture}


%% file: discussions.tex

%
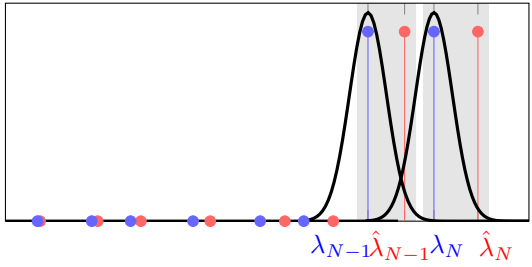
\begin{figure}[t]
    \centering
    \input{plots_stability/high_frequency_feature.tex}      
    \caption{High frequency feature extraction. We illustrate two sharp filters designed to successfully extract high frequency features located at $\lambda_{N-1}$ and $\lambda_{N}$. However, when the graph is slightly perturbed, which results in large changes in high frequency eigenvalues, the designed filters are no longer able to extract these features, now located at $\hlam_{N-1}$ and $\hlam_{N}$, since they have moved out of the narrow pass band of the filter.}
    \label{fig_high_frequency_feature}
\end{figure}

%
\section{Discussions}\label{sec_linear_filters_stability}

From the analysis of model (i) in Prop.~\ref{prop:Delta} we concluded that Lipschitz filters are stable under absolute perturbations, but the stability presents a trade-off with the selectivity of the filters (the more stable the GNN is, the less selective are the filters that compose it). Moreover, we commented that the absolute perturbation model presents certain limitations by not taking into account the underlying graph support.

Under a relative perturbation model, integral Lipschitz filters can be made arbitrarily selective near $\lambda \approx 0$ without sacrificing stability. Therefore, in order to discriminate among signals with frequency content in high values of $\lambda$ we need to spill the information into lower-eigenvalue frequencies, which is easily achieved by the mixing effect of the nonlinearities employed. The following discussion illustrates the intricacies of the stability results put forward in Thm.~\ref{thm:GNNStability} and Prop.~\ref{prop:Delta}.

Suppose that we have shift operators $\bbS$ and $\hbS$ where the latter is a simple scaling of the former by a factor $(1+\varepsilon)$
\begin{equation}\label{eqn:graphDilation}
   \hbS = (1+\varepsilon)\bbS .
\end{equation}
The graph dilation in \eqref{eqn:graphDilation} produces a graph in which all edges are scaled by a $(1+\varepsilon)$ factor. This is a perturbation model of the form in \eqref{eqn_def_relative_perturbation} with $\bbE=(\varepsilon/2)\bbI$. We consider that $\varepsilon \approx 0$ in which case the graph dilation produces a minimal modification of the graph. Note that, for such a perturbation, we have $\delta = 0$ in model (ii) and it also satisfies the structural constraint \eqref{eqn:structuralConstraint} of model (iii), so that both models are applicable here.

Suppose now that we are given a set of filter coefficients $\bbh$ and that we consider the filter $\bbH(\bbS)$ implemented on GSO $\bbS$ vis-\`a-vis the filter $\bbH(\hbS)$ implemented on another GSO $\hbS$ [cf. \eqref{eqn:graphFilter}-\eqref{eqn:graphFilter_perturbed}]. Given that the graph perturbation is inconsequential we would expect the filter differences to be inconsequential as well. Thm.~\ref{thm:filterStabilityRelative} states that if the filters are integral Lipschitz this is true but if they are simply Lipschitz this need not be true. To understand this we look at the differences between the spectra of $\bbS$ and $\hbS$.

%
\begin{figure}[t]
    \centering
    \input{plots_stability/relu_spectrum}
    \caption{Effect of pointwise nonlinearity. Let $\bbx = \bbv_{N}$ be the graph signal with a frequency response $\tbx$ given by $\tdx_{N}=1$ and $\tdx_{i} = 0$ for all $i=1,\ldots,N-1$. Signal $\bbx$ has a single nonzero value located at the highest frequency, making it impossible to be extracted with a stable linear filter. When applying a nonlinearity to this signal, we observe that nonzero frequency components arise throughout the spectrum, spilling the information contained in the highest frequency into lower frequencies. This facilitates the use of a bank of stable linear filters to successfully collect this information at lower frequencies.
    }
    \label{fig_relu_spectrum}
\end{figure}
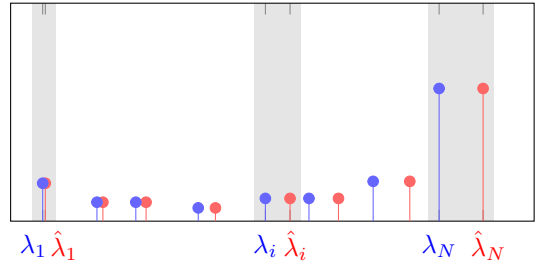

Given that $\bbS$ and $\hbS$ are related by a scaling, they share the same eigenvectors and the scaling is translated to the eigenvalues. Thus, if $\bbS = \bbV\bbLam\bbV^{\Hr}$ is the eigenvector decomposition of $\bbS$ [cf. \eqref{eqn:eigendecomposition}], the eigenvector decomposition of $\hbS$ is
\begin{equation}\label{eqn:graphDilationEigenvalues}
   \hbS = \bbV \big[(1+\varepsilon)\bbLam\big]\bbV^{\Hr}.
\end{equation}
As per \eqref{eqn:graphDilationEigenvalues}, the eigenvalues of $\hbS$ are the eigenvalues of $\bbS$ scaled by a factor $(1+\varepsilon)$. Thus, the effect of the dilation in \eqref{eqn:graphDilation} on a filter with frequency response $h(\lam)$ is that instead of instantiating the response at eigenvalues $\lam_i$ we instantiate it at eigenvalues $(1+\varepsilon)\lam_i$. Consequently the response values that we expect to be $h(\lam_i)$ if the filter is run on $\bbS$ actually turn out to be $h((1+\varepsilon)\lam_i)$ if the filter is run on $\hbS$. This observation is the core argument in the proof of Thm.~\ref{thm:filterStabilityRelative} and motivates the important observations that we discuss next.

%
\begin{figure*}[t]
    \centering
    \begin{subfigure}{0.25\textwidth}
        \centering
        \includegraphics[width=0.9\textwidth]{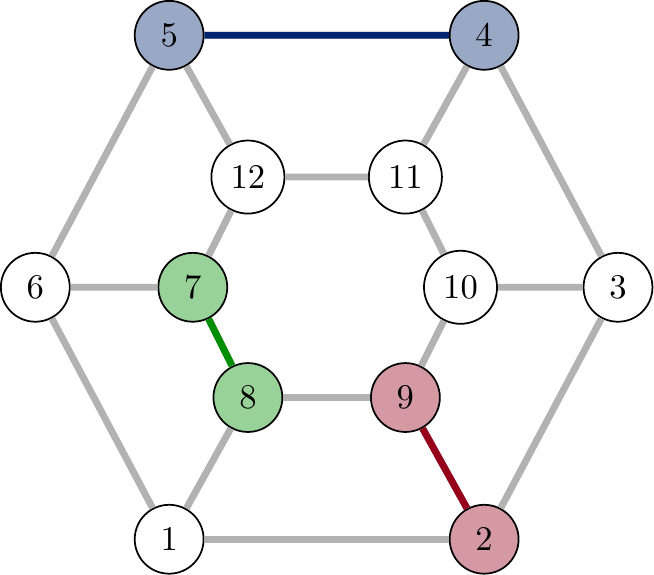}
        \caption{Graph $\ccalG$ and signal $\bbx$}
        \label{original}
    \end{subfigure}
    \hfill
    \begin{subfigure}{0.25\textwidth}
        \centering
        \includegraphics[width=0.9\textwidth]{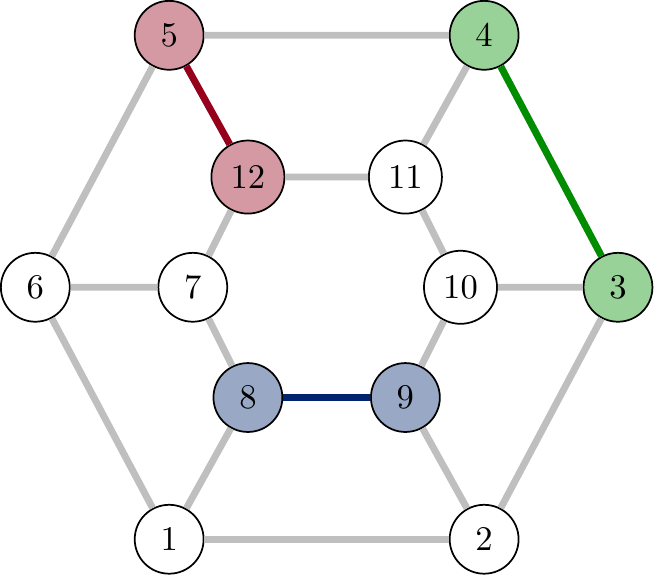} 
        \caption{Graph $\ccalG$ and permuted signal $\bbP^{\Tr}\bbx$}
        \label{symmetries}
    \end{subfigure}
    \hfill
    \begin{subfigure}{0.25\textwidth}
        \centering
        \includegraphics[width=0.9\textwidth]{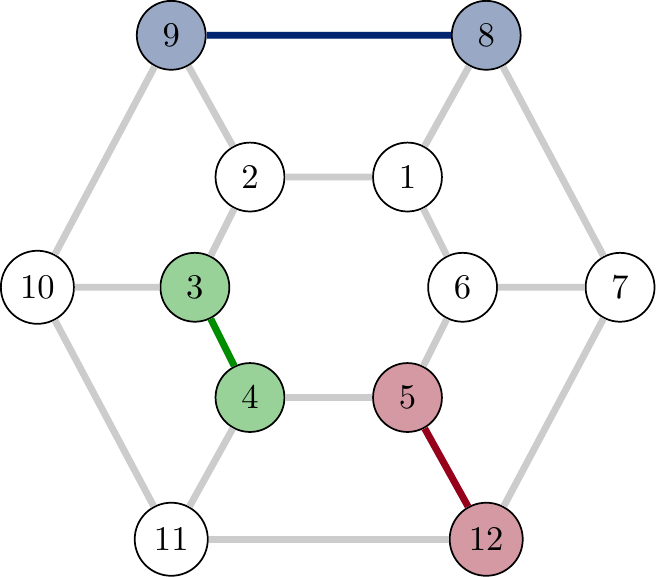} 
        \caption{Permuted graph $\hat{\ccalG}$ and permuted signal $\bbP^{\Tr}\bbx$}
        \label{reordering}
    \end{subfigure}
    \caption{Permutation equivariance of graph neural networks (GNNs). The output of a GNN is equivariant to graph permutations (Proposition \ref{prop:filterPermutationEquivariance}). This not only means independence from labeling but it also shows that GNNs exploit internal signal symmetries. The signals on \subref{original} and \subref{symmetries} are different signals on the same graph but they are permutations of each other -- interchange inner and outer hexagons and rotate $180^{\circ}$ [c.f. \subref{reordering}]. A GNN would learn how to classify the signal in \subref{symmetries} from seeing examples of the signal in \subref{original}. Integers represent the labeling, while colors represent graph signal values.}
    \label{fig:permutationEquivariance}
\end{figure*}

%
\medskip \noindent {\bf Graph perturbations and filter perturbations.} Fig.~\ref{fig_graph_dilation} illustrates the effect of the dilation in \eqref{eqn:graphDilation} on a Lipschitz (top) and integral Lipschitz filter (bottom). The difference in the positions between eigenvalues is given by $\hlam_i-\lam_i = \varepsilon\lam_i$, and as such, depends on the value of the specific eigenvalue $\lambda_{i}$. For low-eigenvalue frequencies $\lam_i$ the dilation results in a small perturbation of the eigenvalues. If the change in eigenvalues is small the change in the filter's response from $h(\lam_i)$ to $h(\hlam_i)$ is small for both filters. For large eigenvalues the difference $\hlam_i-\lam_i = \varepsilon\lam_i$ grows large. For Lipschitz filters a large difference in the arguments may translate into a large difference in the instantiated values of frequency responses $h(\hlam_i)$ and $h(\lam_i)$,
\begin{equation}\label{eqn_discussions_unstable}
   |h(\hlam_i) - h(\lam_i)| 
       \ \approx \ |\hlam_i-\lam_i| 
       \ =       \ \varepsilon \lam_i.
\end{equation}
This explains the filter's instability. A small graph perturbation may result in a large filter perturbation at high-eigenvalue frequencies. For integral Lipschitz filters, on the other hand, changes in the frequency response must taper off as $\lam$ grows. Thus, even though there may be a large variation in the eigenvalues the instances of the frequency responses are close
\begin{equation}\label{eqn_discussions_stable}
   |h(\hlam_i) - h(\lam_i)| 
       \ \approx \ \frac{|\hlam_i-\lam_i|}{|\hlam_i+\lam_{i}|/2} 
       \ =       \ \frac{2\varepsilon}{2-\varepsilon} \approx \varepsilon.
\end{equation}
This explains the filter's stability. No matter how large the eigenvalues are, a small perturbation of the graph results in a small perturbation of the graph filter. Thm.~\ref{thm:filterStabilityRelative} shows that this is true for arbitrary relative perturbations.

%
\medskip \noindent {\bf Graph perturbations and feature identification.} There is an obvious cost we pay for the stability of integral Lipschitz filters: they are unable to discriminate high-eigenvalue frequencies. The graph dilation example shows that this is not a limitation of the analysis. It is impossible to have a filter that is both stable and able to isolate high-eigenvalue features because small graph perturbations can result in large eigenvalue perturbations. This is a major drawback of linear graph filters in the extraction of features from graph signals. To illustrate this drawback suppose we have graph signals $\bbx_{1} = \bbv_N$ and $\bbx_2 = \bbv_{N-1}$ and we want to design graph filters to discriminate between the two. The graph frequency domain representation of these two signals on the graph $\bbS$ are shown in Fig.~\ref{fig_high_frequency_feature}. For us to discriminate between $\bbx_{1} = \bbv_N$ and $\bbx_2 = \bbv_{N-1}$ we need filters centered at frequencies $\lam_{N}$ and $\lam_{N-1}$. These filters must have sharp transitions so that the filter isolating $\bbx_1 = \bbv_{N}$ does not let the signal $\bbx_{2} = \bbv_{N-1}$ pass and, conversely, the filter isolating $\bbx_2 = \bbv_{N-1}$ does not let the signal $\bbx_{1} = \bbv_N$. Yet, if these filters are sharp on large eigenvalues, they will be unstable. More specifically, let $\hlam_{N} = (1+\varepsilon)\lambda_{N}$ be the eigenvalue associated to $\bbx_{1} = \bbv_{N}$ in the perturbed graph, and $\hlam_{N-1} = (1+\varepsilon)\lambda_{N-1}$ be the one associated to $\bbx_{2} = \bbv_{N-1}$. Now, since the filters were designed to be sharp around $\lambda_{N}$ and $\lambda_{N-1}$, but the perturbed eigenvalues $\hlam_{N}$ and $\hlam_{N-1}$ are far from these (at points where the filter response is virtually zero) the filter fails to adequately recover $\bbx_{1}$ and $\bbx_{2}$ in the perturbed graph. See Fig.~\ref{fig_high_frequency_feature} for an illustration of the instability effect at large eigenvalues.

%
\medskip \noindent {\bf Pointwise nonlinearities.} So far, we have observed that stable filters require a flat response on high-eigenvalue frequencies, but that this inevitably prevents them from discriminating between features located at these frequencies. This illustrates an inherent, insurmountable limitation of linear information processing schemes. Neural networks introduce pointwise nonlinearities to the processing pipeline, as a computationally straightforward means of discriminating information located at large eigenvalues. The basic effect of these nonlinearities is to cause a spillage of information throughout the frequency band, see Fig.~\ref{fig_relu_spectrum}. This spillage of information into smaller eigenvalues allows for a stable filter to accurately discriminate between them, since information at these frequencies does not get severely affected by perturbations. However, since the energy in smaller eigenvalues is usually less than the energy still found at larger ones, and since it is also spread through a wide band of frequencies, the use of a bank of linear filters becomes a sensitive idea to better capture this spillage. Therefore, the use of banks of linear filters in combination with pointwise nonlinearities allows for information processing architectures that are able to capture high-eigenvalue frequency content in a stable fashion.

\medskip \noindent {\bf Permutation Equivariance.} The permutation equivariance stated in Prop.~\ref{prop:GNNPermutationEquivariance} shows that the features that are learned by a GNN are independent of the labeling of the graph. But permutation equivariance is also important because it means that GNNs exploit internal signal symmetries as we illustrate in Fig.~\ref{fig:permutationEquivariance}. The graphs in Figs.~\ref{original} and \ref{symmetries} are the same, as indicated by the integer labels. The signals in Figs.~\ref{original} and \ref{symmetries} are different, as indicated by different colors. However, it is possible to permute the graph onto itself to make the signals match -- rotate $180^{\circ}$ degrees and pull it inside out (Fig.~\ref{reordering}). It then follows from Prop.~\ref{prop:GNNPermutationEquivariance} that the output of a GNN applied to the signal on the left (\ref{original}) is a corresponding permutation of the output of the same GNN applied to the signal on the right (\ref{symmetries}). This is beneficial because we can learn to process the signal on (\ref{original}) from seeing examples of the signal on (\ref{symmetries}). We note that, while most graphs do not exhibit perfect symmetries, they might have (sub)structures that are close to permutations. Therefore, permutation equivariance shows the ability of GNNs to exploit these similarities.

%

%% file: plots_stability/high_frequency_feature.tex

\def \thisplotscale {2.9}
\def \unit {\thisplotscale cm}

\def \frequencyresponse 
     {1.1*exp(-(2.5*(x-6.048))^2}
\def \frequencyresponsetwo 
     {1.1*exp(-(2.5*(x-5.04))^2}

\hspace{-2.9mm}
\begin{tikzpicture}[x = 1*\unit, y=1*\unit]

\def \factorx {2.4/8}
\def \deltax  {0.5*\factorx}
\def \shadeshift  {0.05}

\path [fill=black!20, opacity = 0.5] 
              (\deltax + 6.048*\factorx - \shadeshift, 0.00) rectangle 
              (\deltax + 6.720*\factorx + \shadeshift, 1.00);

\path [fill=black!20, opacity = 0.5] 
              (\deltax + 5.04*\factorx - \shadeshift, 0.00) rectangle 
              (\deltax + 5.60*\factorx + \shadeshift, 1.00);

\begin{axis}[scale only axis,
             width  = 2.4*\unit,
             height = 1*\unit,
             xmin = -0.5, xmax=7.5,
             xtick = {5.60, 5.04, 6.72, 6.048},
             xticklabels = {\red{$\qquad\hlam_{N-1}\phantom{\lam_{N-1}}$},
                            \blue{$\lam_{N-1}\qquad  $}, 
                            \red{$\qquad\hlam_{N}\phantom{\lam}$},
                            \blue{$\quad\lam_N$}},
             ymin = -0, ymax = 1.15,
             ytick = {-1},
             typeset ticklabels with strut,
             enlarge x limits=false]

\addplot+[samples at = {0.03, 0.91, 1.57, 
                        2.63, 3.77, 4.51}, 
          color = red!60, 
          ycomb, 
          mark=otimes*, 
          mark options={red!60}]
         {0};

\addplot+[samples at = {6.72, 5.60}, 
          color = red!60, 
          ycomb, 
          mark=otimes*, 
          mark options={red!60}]
         {1};

\addplot+[samples at = {-0.01, 0.819, 1.413, 
                        2.367, 3.393, 4.059}, 
          color = blue!60, 
          ycomb, 
          mark=oplus*, 
          mark options={blue!60}]
         {0};

\addplot+[samples at = {6.048, 5.04}, 
          color = blue!60, 
          ycomb, 
          mark=oplus*, 
          mark options={blue!60}]
         {1};

\addplot[ domain=-0.5:5.5, 
          samples = 2, 
          color = black,
          line width = 1.2]
         {0};

\addplot[ domain=5.0:7.5, 
          samples = 70, 
          color = black,
          line width = 1.2]
         {\frequencyresponse};

\addplot[ domain=4.0:7.0, 
          samples = 70, 
          color = black,
          line width = 1.2]
         {\frequencyresponsetwo};
\addplot[ domain=7.0:7.5, 
          samples = 2, 
          color = black,
          line width = 1.2]
         {0};

\end{axis}
\end{tikzpicture}


%% file: plots_stability/relu_spectrum.tex

\def \thisplotscale {2.9}
\def \unit {\thisplotscale cm}

\def \frequencyresponse 
     {   0.8*exp(-(1*(x-1.2))^2) 
       + 0.7*exp(-(0.7*(x-4))^2) 
       + 0.8*exp(-(1.4*(x-6))^2) 
       + 0.1}

\hspace{-2.9mm}
\begin{tikzpicture}[x = 1*\unit, y=1*\unit]

\def \factorx {2.4/8}
\def \deltax  {0.5*\factorx}
\def \shadeshift  {0.05}

\path [fill=black!20, opacity = 0.5] 
              (\deltax - 0.001*\factorx - \shadeshift, 0.00) rectangle 
              (\deltax + 0.030*\factorx + \shadeshift, 1.00);
\path [fill=black!20, opacity = 0.5] 
              (\deltax + 3.393*\factorx - \shadeshift, 0.00) rectangle 
              (\deltax + 3.770*\factorx + \shadeshift, 1.00);
\path [fill=black!20, opacity = 0.5] 
              (\deltax + 6.048*\factorx - \shadeshift, 0.00) rectangle 
              (\deltax + 6.720*\factorx + \shadeshift, 1.00);

\begin{axis}[scale only axis,
             width  = 2.4*\unit,
             height = 1*\unit,
             xmin = -0.5, xmax=7.5,
             xtick = {0.03, -0.01, 3.77, 3.393, 6.72, 6.048},
             xticklabels = {\red{$\qquad\hlam_1\phantom{\lam}$},
                            \blue{$\lam_1\ \ $}, 
                            \red{$\quad\hlam_i\phantom{\lam}$}, 
                            \blue{$\lam_i$},
                            \red{$\quad\hlam_{N}\phantom{\lam}$},
                            \blue{$\lam_N$}},
             ymin = -0, ymax = 1.15,
             ytick = {-1},
             typeset ticklabels with strut,
             enlarge x limits=false]

\addplot+[color = red!60, 
          ycomb, 
          mark=otimes*, 
          mark options={red!60}]
          coordinates { (0.03, 0.20)
                        (0.91, 0.10)
                        (1.57, 0.10)
                        (2.63, 0.07)
                        (3.77, 0.12)
                        (4.51, 0.12)
                        (5.60, 0.21)
                        (6.72, 0.70)};

\addplot+[samples at = {-0.01, 0.819, 1.413, 
                        2.367, 3.393, 4.059, 
                        5.04, 6.048}, 
          color = blue!60, 
          ycomb, 
          mark=oplus*, 
          mark options={blue!60}]
          coordinates { (-0.010, 0.20)
                        ( 0.819, 0.10)
                        ( 1.413, 0.10)
                        ( 2.367, 0.07)
                        ( 3.393, 0.12)
                        ( 4.059, 0.12)
                        ( 5.040, 0.21)
                        ( 6.048, 0.70)};

\end{axis}
\end{tikzpicture}

%% file: experimentsStability.tex

\begin{figure*}
    \centering
    \begin{subfigure}{.33\textwidth}
        \centering
        \includegraphics[width=0.95\textwidth]{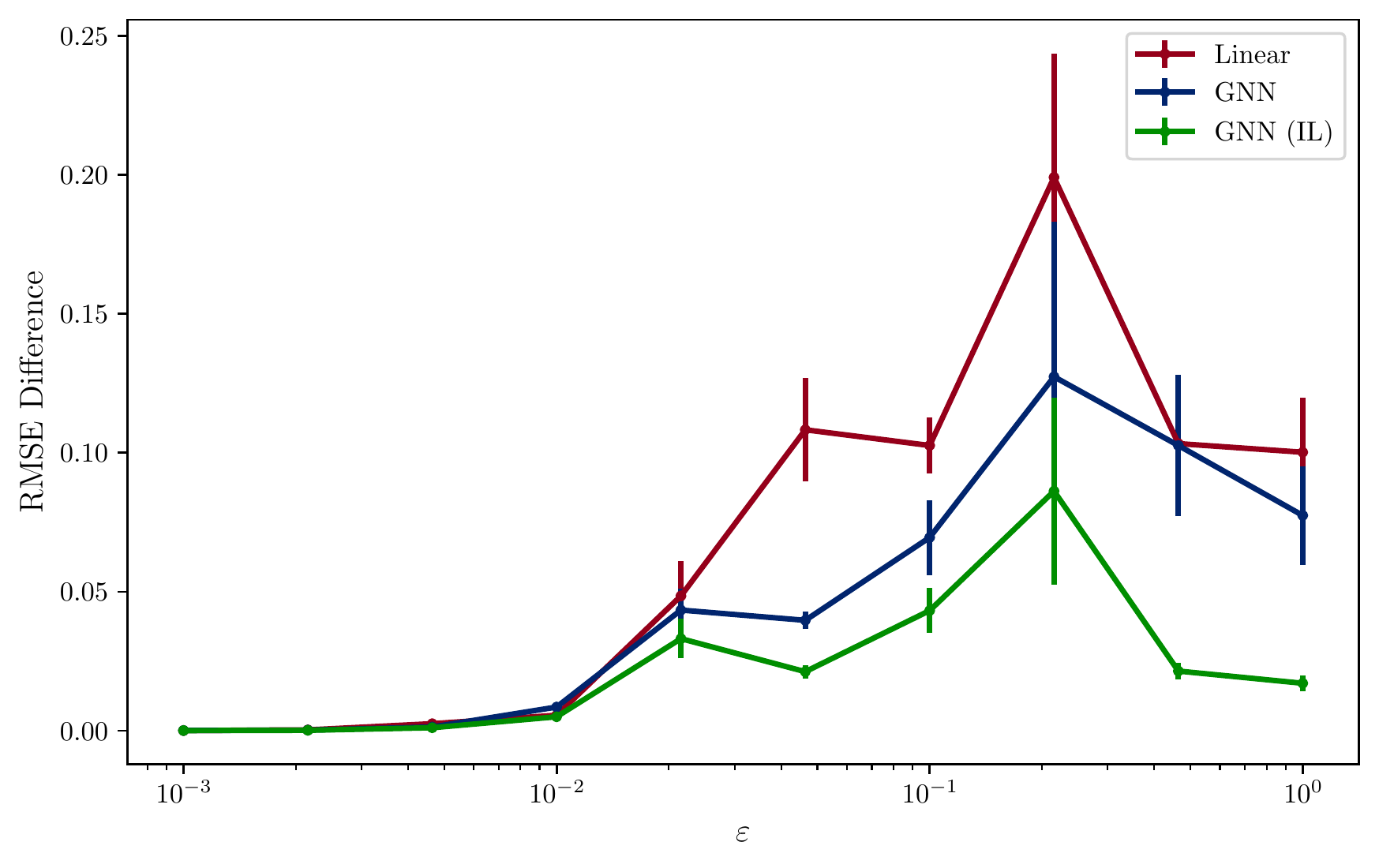}
        \caption{}
        \label{subfig:costDiff}
    \end{subfigure}%
    \hfill
    \begin{subfigure}{.33\textwidth}
        \centering
        \includegraphics[width=0.95\textwidth]{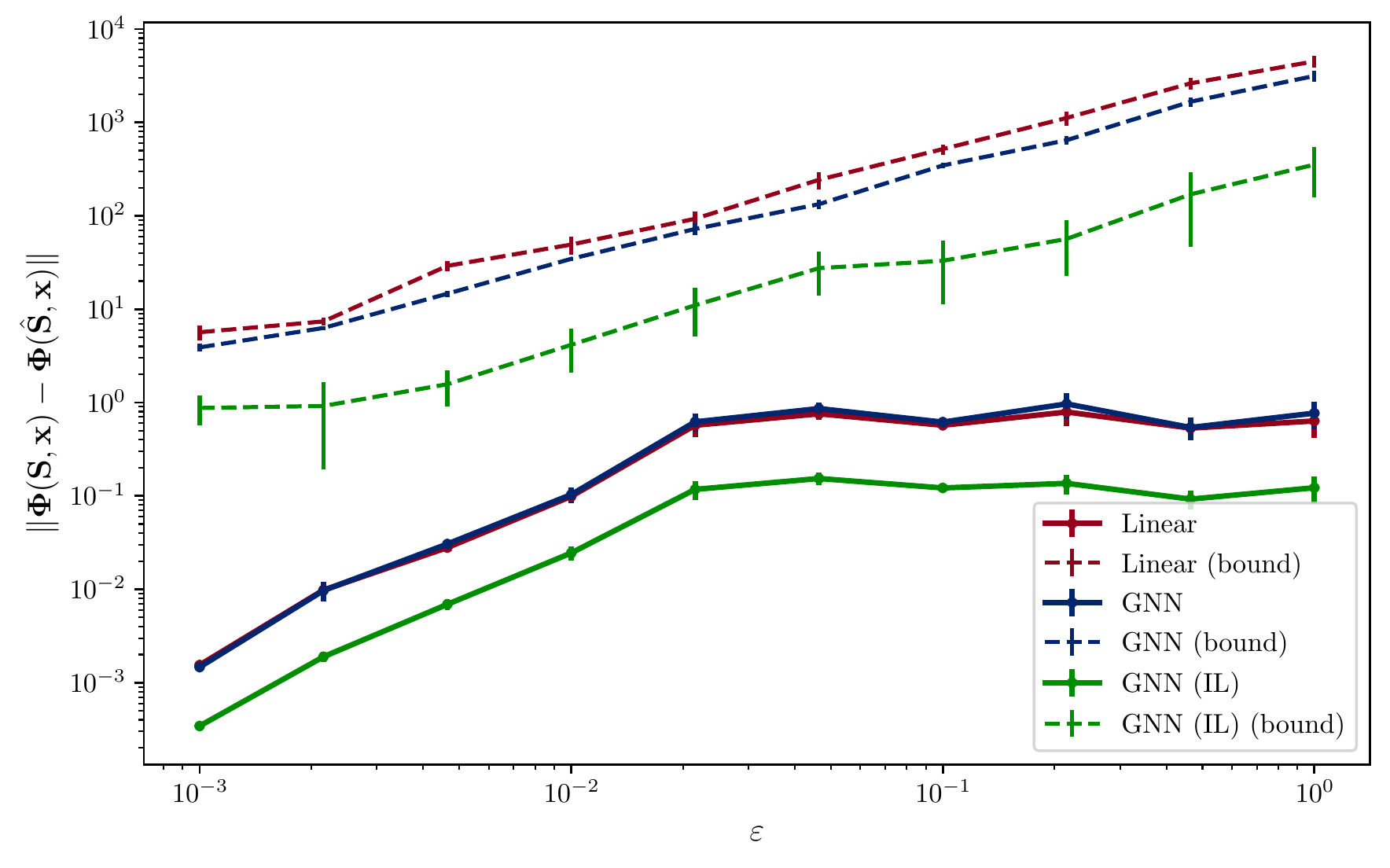}
        \caption{}
        \label{subfig:GNNdiff}
    \end{subfigure}%
    \hfill
    \begin{subfigure}{.33\textwidth}
        \centering
        \includegraphics[width=0.95\textwidth]{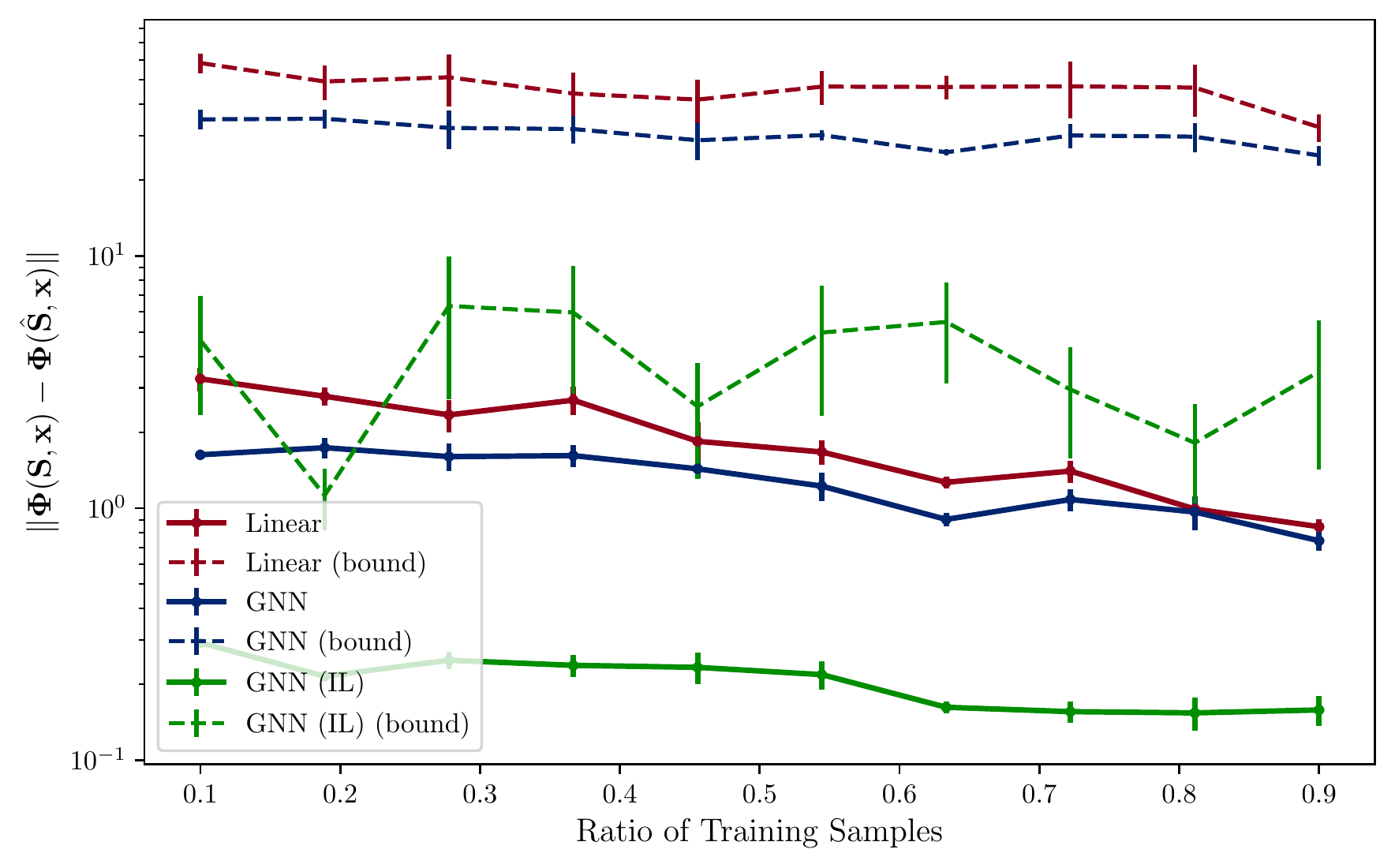}
        \caption{}
        \label{subfig:GNNdiffRatio}
    \end{subfigure}%
    \caption{Movie Recommendation problem. The baseline evaluation performance (RMSE) is $0.84 (\pm 0.15)$ for the linear architecture (Linear), $0.84 (\pm 0.16)$ for the GNN that learned from the space of all graph filters (GNN), and $0.83 (\pm 0.14)$ for the GNN that learned integral Lipschitz filters (GNN (IL)). Synthetic Experiment:  \subref{subfig:costDiff} Change in evaluation measure (RMSE) due to synthetic changes in the underlying graph support, where we observe that the GNN (IL) is more stable than the GNN and the Linear architectures; \subref{subfig:GNNdiff} Change in the output of the GNN due to synthetic changes in the underlying graph support, we observe that the GNN (IL) is consistently more stable, and that the bounds are not tight. Estimation error experiment. \subref{subfig:GNNdiffRatio} Changes in the output of the GNN due to changes in the estimation of the graph support, stemming from using different sizes of training set; again, we observe that GNN (IL) is consistently more stable than the other two architectures and that the bounds are not tight.}
    \label{fig:results}
\end{figure*}


\section{Numerical Experiments} \label{sec:sims}

To illustrate the GNN stability results in a practical setting, we consider the problem of movie recommendation systems \cite{Huang18-RatingGSP}. We describe the problem with a graph where each node is a movie, and each edge weight represents the rating similarity for each pair of movies. The information from each user is modeled as a graph signal, whereby the value assigned to each node represents the rating the user has given to each movie watched. The objective is to infer the rating a user would give to a specific, unseen movie, based on the ratings given to the other movies, and the rating similarities present in the graph structure. We carry out two experiments, the first one showing the stability under a synthetic, controlled relative perturbation; and the second one considering a more realistic perturbation arising from the error in estimating the underlying graph structure. The main objectives of this section are to illustrate how \emph{loose} the bound actually is and also that GNNs using integral Lipschitz filters are more stable (which amounts to showing the effect of $C$).

\myparagraph{Dataset.} We use the MovieLens-100k dataset \cite{Harper16-MovieLens}. This dataset contains $100,000$ ratings given by $943$ users to some of the $1,582$ movies available. Ratings go from $1$ indicating a disliked movie, to $5$ indicating a liked movie. The movie \emph{Star Wars} is selected as the target movie to estimate the rating, since it is the movie with the largest number of available ratings.

\myparagraph{Graph signal processing formulation.} We use $90\%$ of the ratings as part of the training set in order to build the graph support. Each node in the graph is a movie, amounting to $1,582$ nodes. The edge weights are obtained by estimating the Pearson correlation coefficient between each pair of movies as in \cite[eq. (6)]{Huang18-RatingGSP}, based on the ratings contained in the training set only. A $10$ nearest-neighbor graph is built from these edge weights. Once the graph is built, we consider the users that have rated the target movie, which amounts to $583$ users. Each of these users is considered as a graph signal, where each node value is the rating given to that movie. Movies not rated by each user are assigned a $0$. The rating given to the target movie is extracted as a label, and zeroed out in the graph signal. This dataset of $583$ graph signals and the corresponding labels is split into $90\%$ for the training set and $10\%$ for the testing set, with the training set further split into $10\%$ for validation, and the rest for training.

\myparagraph{Architectures.} We consider two GNN architectures as mapping parametrizations between the input graph signal (the ratings given to some of the movies) and the label (the rating given to the target movie). The architectures have a single-layer GNN \eqref{eqn_gnn_filter}-\eqref{eqn_gnn_nonlinearity} with $F_{0}=1$ input feature (the rating value) and $F_{1} = 64$ output features, and $5$ filter taps. The nonlinearity used is a ReLU. One of the architectures, labeled as `GNN', learns from the space of all graph filters, while another one, labeled as `GNN (IL)', learns only integral Lipschitz filters. We compare these two architectures with a learned linear graph filter with $64$ output features and $5$ filter taps. All architectures have a local, linear readout layer mapping the $64$ features at the target node to a single scalar that estimates the rating (i.e. a learnable $64 \times 1$ matrix which is equivalent to a second layer with $F_{2}=1$ and $1$ filter tap).

\myparagraph{Training and evaluation.} We train all the architectures by minimizing a smooth $L1$ loss between the estimated rating at the output of the readout layer and the extracted label. We use an ADAM optimizer with learning rate $0.005$ and forgetting factors $0.9$ and $0.999$. We train for $40$ epochs with batches of size $5$. The evaluation performance is the root mean squared error (RMSE) as is standard in the movie recommendation problem \cite{Huang18-RatingGSP}. In all cases, we run $5$ random dataset partitions, and report the average performance across these realizations, as well as the standard deviation.

\myparagraph{Experiments.} We run two experiments. For the first experiment, we consider synthetic relative perturbations, and analyze how the output of the GNN and the evaluation performance change under controlled perturbations of the graph support at test time. For the second experiment, we consider a real world perturbation stemming from different construction of the graph support. That is, note that the graph support is build out of the training set, so changing the training set would lead to different graph support, each reflecting a different estimation of the Pearson coefficient. In particular, smaller training sets would lead to larger estimation error, and thus, by changing the ratio of the training/testing set split, we can adjust the estimation error. As a matter of fact, we note in practice that using smaller sets to build the graph leads to a larger relative perturbation. In the second experiment we analyze the stability of the GNNs for different values of the training set ratio, analyzing the usefulness of stable architectures in a real world setting with inference estimation errors \cite{Segarra17-Template}.

\myparagraph{Synthetic experiment.} We generate a random perturbation matrix $\bbE$ such that $\|\bbE\| \leq \varepsilon$ and \eqref{eqn:structuralConstraint} are satisfied. We do so by generating a diagonal matrix $\bbE$, with diagonal elements drawn uniformly at random from the interval $[(1-\varepsilon) \varepsilon, \varepsilon])$. Note that such a perturbation is not a simple edge dilation like the one discussed in Sec.~\ref{sec_linear_filters_stability}. We then build the perturbed matrix $\hbS=\bbS + \bbE \bbS + \bbS \bbE$. Note that we do not recompute the $10$ nearest neighbors. We control the perturbation size $\varepsilon$ and simulate it from $10^{-3}$ to $1$. The change in the evaluation measure (the change in the RMSE) can be found in Fig.~\ref{subfig:costDiff}. We observe that, for small $\varepsilon$, there is virtually no change in the output between all three architectures, but as $\varepsilon$ grows, the change in the GNN with integral Lipschitz filters is smaller than the change in both the linear and the GNN architectures. This evidences that the GNN with integral Lipschitz filters is indeed more stable. In Fig.~\ref{subfig:GNNdiff} we show specifically the change in the output of the GNN layer caused by the perturbation. We also show the the bounds. First, we note that the GNN that learned integral Lipschitz filters is consistently more stable. We also note that the bounds are not tight bounds, essentially because the bound on the eigenvectors is valid for all graphs and thus is not tight.

\myparagraph{Estimation error experiment.} In this last experiment, we consider a more realistic perturbation. We consider architectures trained on a graph based on Pearson correlations estimated from a $90\%$ split of the training set. Then, at test time, we consider architectures running on graphs with Pearson correlations estimated from smaller training sets, ranging from $10\%$ to $90\%$. Since the number of training sets are smaller, then the estimation error of the graph is larger. This is a scenario that arises when the underlying graph support is not known and needs to be estimated (so we can consider $\bbS$ to be the true support, and $\hbS$ to be the estimation) \cite{Segarra17-Template}. The change at the output of the GNN layer is shown in Fig.~\ref{subfig:GNNdiffRatio}. We see that the GNN with integral Lipschitz filters is approximately one order of magnitude more stable than the GNN trained with arbitrary graph filters. This one, in turn, is slightly more stable than the linear architecture. Likewise, we show the bounds and see that they are not tight.

%% file: conclusionsStability.tex


\section{Conclusions} \label{sec:conclusions}

We focused on the impact that changes in the underlying topology have on the output of a GNN. First, we studied changes brought by permutations. We proved that GNNs are permutation equivariant, and that this implies that they effectively exploit the topological symmetries present in the underlying graph. Then, we discussed the absolute perturbation model existing in the literature, and proved that GNNs composed of Lipschitz filters are stable. However, not only the absolute perturbation model ignores the particularities of the underlying graph, but also the stability comes at the expense of the discriminative power of the filters (i.e. the more stable, the less discriminative). We thus proposed a relative perturbation model and proved that filters used in GNNs need be \emph{integral} Lipschitz for the resulting architecture to be stable. Integral Lipschitz filters can be made arbitrarily selective around low-eigenvalue frequencies, but need to have a flat response in high-eigenvalue frequencies, precluding accurate discrimination of information located in this band. We show that the frequency mixing effect of nonlinearities succeeds in spreading the information throughout the frequency spectrum, and thus allowing for accurate discrimination of information located at all frequencies. In essence, superior performance of GNNs can be explained by the fact that they are both stable and discriminative architectures, whereas linear graph filters can only satisfy one of these properties. We illustrated the discriminability and stability properties of both GNNs and graph filters in a movie recommendation problem. It was observed in the experiments that the bounds are not tight, and thus they can be improved. One of the reasons for this lack of tightness is that the bound in Theorem~\ref{thm:GNNStability} holds for all possible perturbations, resulting in a rather large value of the vector misalignment constant. This bound can certainly be improved if more specific perturbation models are studied, giving raise to particular bounds for the constant. This is envisioned as a future area of research, where the bounds provided herein are improved for specific applications on specific graphs. Likewise, Theorem~\ref{thm:GNNStability} holds for perturbations that have the same number of nodes as the original graph. Extending this result to graphs of different size is an active area of research.

%% file: proofsStability.tex



\section{Permutation Equivariance of Graph Filters}\label{sec_apx_A}

\begin{proof}[Proof of Prop.~\ref{prop:filterPermutationEquivariance}]
A permutation matrix $\bbP \in \ccalP$ is an orthogonal matrix, $\bbP^{\Tr} \bbP = \bbP \bbP^{\Tr} = \bbI$, from where it follows that powers $\hbS^k$ of a permuted shift operator are permutations of the respective shift operator powers $\bbS^k$
\begin{equation}\label{eqn_prop_equivariance_pf_10}
    \hbS^{k} = (\bbP^{\Tr} \bbS \bbP)^k
             = \bbP^{\Tr} \bbS^{k} \bbP.
\end{equation}
Substituting this fact in the definition of the permuted graph filter $\bbH(\hbS)$ in \eqref{eqn:graphFilter_perturbed} yields
\begin{equation}\label{eqn_prop_equivariance_pf_20}
    \bbH(\hbS) 
        = \sum_{k} h_{k} \left( \bbP^{\Tr} \bbS^{k} \bbP \right) 
        =  \bbP^{\Tr}\bigg( \sum_{k} h_{k}  \bbS^{k}  \bigg) \bbP.
\end{equation}
In the last equality the sum is the filter $\bbH(\bbS) = \sum_{k} h_{k}  \bbS^{k}$ as defined in \eqref{eqn:graphFilter}. We can then write $\bbH(\hbS) = \bbP^{\Tr} \bbH(\bbS) \bbP$ and use this fact to express application of the permuted filter $\bbH(\hbS)$ to the permuted signal $\hbx = \bbP^{\Tr} \bbx$ as
\begin{equation}\label{eqn_prop_equivariance_pf_30}
    \hbz = \bbH(\hbS) \hbx
        = \bbP^{\Tr} \bbH(\bbS) \bbP \bbP^{\Tr} \bbx
\end{equation}
Since $\bbP$ is orthogonal we have that $\bbP \bbP^{\Tr}=\bbI$. Substituting this into the right hand side of \eqref{eqn_prop_equivariance_pf_30}, the result in \eqref{eqn_prop_filter_permutation_equivariance} follows.
\end{proof}


\section{Stability under Absolute Perturbations}\label{sec_apx_B}

\begin{lemma} \label{l:Evi}
    Let $\bbS = \bbV \bbLambda \bbV^{\Hr}$ and $\bbE = \bbU \bbM \bbU^{\Hr}$ such that $\|\bbE\| \leq \varepsilon$. For any eigenvector $\bbv_{i}$ of $\bbS$ it holds that
    \begin{equation}
    \bbE \bbv_{i} = m_{i} \bbv_{i}+ \bbE_{U} \bbv_{i}
    \end{equation}
    with $\|\bbE_{U}\| \leq \varepsilon \delta$, where $\delta = (\| \bbU - \bbV\|^{2} + 1)^{2} - 1$.
\end{lemma}

\begin{proof}
    Start by writing the error matrix $\bbE$ as
    \begin{align}
    \bbE & = \bbE_{V} + \bbE _{U} \\
    \bbE_{V} & = \bbV \bbM \bbV^{\Hr} \\
    \bbE_{U} 
    & = \left( \bbU - \bbV \right) \bbM \left( \bbU - \bbV \right)^{\Hr} \\
    & \qquad 
    + \bbV \bbM \left( \bbU - \bbV \right)^{\Hr}
    + \left( \bbU - \bbV \right) \bbM \bbV^{\Hr}. \nonumber
    \end{align}
    We see that $\bbE_{V} \bbv_{i} = m_{i} \bbv_{i}$ since $\bbv_{i}$ is an eigenvector of $\bbE_{V}$. Next, note that, since $\|\bbE\| \leq \varepsilon$, then $\|\bbM \| \leq \varepsilon$, so that
    \begin{equation} \label{eqn:boundEu}
    \begin{aligned}
    \|\bbE_{U}\| & \leq
    \left\| (\bbU - \bbV) \bbM (\bbU - \bbV)^{\Hr} \right\| \\
    & \quad 
    + \left\| \bbV \bbM (\bbU - \bbV)^{\Hr} \right\| 
    + \left\| (\bbU - \bbV) \bbM \bbV^{\Hr} \right\| \\
    & \leq 
    \| \bbU - \bbV \|^{2} \| \bbM \| 
    + 2 \| \bbU - \bbV \| \| \bbV \| \| \bbM \| \\
    & \leq 
    \varepsilon \| \bbU - \bbV \|^{2} 
    + 2 \varepsilon \| \bbU - \bbV \| \\
    & = \varepsilon \left( (\| \bbU - \bbV\|^{2} + 1)^{2} - 1 \right) = \varepsilon \delta
    \end{aligned}
    \end{equation}
    which completes the proof.
\end{proof}

\begin{proof}[Proof of Thm.~\ref{thm:filterStabilityAbsolute}]
Since graph filters are permutation equivariant (Prop.~\ref{prop:filterPermutationEquivariance}), we can assume, without loss of generality, that $\bbP_{0} = \bbI$ in \eqref{eqn_def_operator_distance_error_matrix}, writing $\hbS = \bbS + \bbE$. Let us start by computing the first order expansion of $(\bbS + \bbE)^{k}$
\begin{equation} \label{eqn:matrixTaylor}
    (\bbS + \bbE)^{k} 
        = \bbS^{k} 
            + \sum_{r=0}^{k-1} \bbS^{r} \bbE \bbS^{k-r-1} 
            + \bbC
\end{equation}
with $\bbC$ such that $\|\bbC\| \leq \sum_{r=2}^{k} \binom{k}{r} \|\bbE\|^{r} \|\bbS\|^{k-r}$. Using this first-order approximation back in \eqref{eqn:graphFilter}, we get
\begin{equation} \label{eqn:firstOrderAbsolute}
    \bbH(\hbS) - \bbH(\bbS) = 
        \sum_{k=0}^{\infty} h_{k} 
            \sum_{r=0}^{k-1} 
                \bbS^{r} \bbE \bbS^{k-r-1} 
        + \bbD
\end{equation}
with $\bbD$ such that $\|\bbD\| = \ccalO(\|\bbE\|^{2})$ since the coefficients $\{h_{k}\}_{k=0}^{\infty}$ of the filter stem from the power series expansion of the analytic function $h$ which has bounded derivatives.

Next, consider an arbitrary graph signal $\bbx$ with finite energy $\|\bbx\|<\infty$ that has a GFT given by $\tbx = [\tdx_{1},\ldots,\tdx_{N}]^{\Tr}$ so that $\bbx = \sum_{i=1}^{N} \tdx_{i} \bbv_{i}$ for $\{\bbv_{i}\}_{i=1}^{N}$ the eigenvector basis of the GSO,
\begin{equation}\label{eqn:filterDifferenceWithXAbsolute}
\begin{aligned}
    & \left[ \bbH(\hbS) - \bbH(\bbS) \right] \bbx = \sum_{i=1}^{N} \tdx_{i} \bbD \bbv_{i}
          \\
    & \qquad
        + \sum_{i=1}^{N} \tdx_{i} 
            \sum_{k=0}^{\infty} h_{k} 
                \sum_{r=0}^{k-1} \bbS^{r} \bbE \bbS^{k-r-1} \bbv_{i} .
\end{aligned}
\end{equation}
Let us focus on the second term of the sum in \eqref{eqn:filterDifferenceWithXAbsolute}. It is immediate that $\bbS^{k-r-1}\bbv_{i} = \lambda_{i}^{k-r-1}\bbv_{i}$, so that
\begin{equation} \label{eqn:allSumEvi}
\begin{aligned}
    & \sum_{i=1}^{N} \tdx_{i} 
        \sum_{k=0}^{\infty} h_{k} 
            \sum_{r=0}^{k-1} \bbS^{r} \bbE \bbS^{k-r-1} \bbv_{i} \\
    & \qquad 
        = \sum_{i=1}^{N} \tdx_{i} 
        \sum_{k=0}^{\infty} h_{k} 
            \sum_{r=0}^{k-1} \lambda_{i}^{k-r-1} \bbS^{r} \bbE  \bbv_{i}
\end{aligned}
\end{equation}
Now, using Lemma~\ref{l:Evi} in \eqref{eqn:allSumEvi} yields two terms
\begin{align}
    & \sum_{i=1}^{N} \tdx_{i} 
        \sum_{k=0}^{\infty} h_{k} 
            \sum_{r=0}^{k-1} \lambda_{i}^{k-r-1} \bbS^{r} \bbE  \bbv_{i} \\
    & = \sum_{i=1}^{N} \tdx_{i} 
        \sum_{k=0}^{\infty} h_{k}
            \sum_{r=0}^{k-1} \lambda_{i}^{k-r-1} \bbS^{r} m_{i}  \bbv_{i} 
                \label{eqn:allSumEviFirstTerm} \\
    & \quad + \sum_{i=1}^{N} \tdx_{i} 
            \sum_{k=0}^{\infty} h_{k}
                \sum_{r=0}^{k-1} \lambda_{i}^{k-r-1} \bbV \bbLam^{r} \bbV^{\Hr} \bbE_{U}  \bbv_{i}.
                \label{eqn:allSumEviSecondTerm}
\end{align}
For \eqref{eqn:allSumEviFirstTerm} we note that $\bbS^{r} \bbv_{i} = \lambda_{i}^{r} \bbv_{i}$, leading to the product $\lambda_{i}^{k-r-1} \lambda_{i}^{r} = \lambda_{i}^{k-1}$ being independent of $r$, so that
\begin{equation} \label{eqn:firstTermSumAbsolute}
    \sum_{i=1}^{N} \tdx_{i} m_{i}
        \sum_{k=1}^{\infty} k h_{k} \lambda_{i}^{k-1}  \bbv_{i}
    = \sum_{i=1}^{N} \tdx_{i} m_{i} h'(\lambda_{i}) \bbv_{i}
\end{equation}
where $h'(\lambda_{i}) = \sum_{k=1}^{\infty} k h_{k} \lambda_{i}^{k-1}$ is the derivative $h'(\lambda)$ of $h(\lambda)$ evaluated at $\lambda = \lambda_{i}$. In the case of \eqref{eqn:allSumEviSecondTerm} we note that
\begin{equation}\label{eqn:secondTermSumAbsolute}
\begin{aligned} 
    & \sum_{i=1}^{N} \tdx_{i} \bbV
        \sum_{k=0}^{\infty} h_{k} 
            \sum_{r=0}^{k-1} 
                \lambda_{i}^{k-r-1} \bbLambda^{r} \bbV^{\Hr} \bbE_{U} \bbv_{i}  \\
    & \qquad = \sum_{i=1}^{N} \tdx_{i} \bbV
        \diag(\bbg_{i}) \bbV^{\Hr} \bbE_{U} \bbv_{i}
\end{aligned}
\end{equation}
where $\bbg_{i} \in \reals^{N}$ is such that
\begin{equation}
    [\bbg_{i}]_{j} 
        = \sum_{k=0}^{\infty} h_{k} 
            \sum_{r=0}^{k-1} \lambda_{i}^{k-r-1} \lambda_{j}^{r}.
\end{equation}
For $j=i$ we have $[\bbg_{i}]_{i} = h'(\lambda_{i})$ while, for $j \neq i$, recall that $\sum_{r=0}^{k-1} \lambda_{i}^{k-r-1}\lambda_{j}^{r} = (\lambda_{i}^{k} - \lambda_{j}^{k})/(\lambda_{i} - \lambda_{j})$ so that
\begin{equation} \label{eqn:giAbsolute}
    [\bbg_{i}]_{j} = 
        \begin{cases}
            h'(\lambda_{i}) 
                & \text{ if} j=i \\
            \frac{h(\lambda_{i})-h(\lambda_{j})}{\lambda_{i}-\lambda_{j}} 
                & \text{ if} j \neq i
        \end{cases}.
\end{equation}
Note $\max_{j} |[\bbg_{i}]_{j}| \leq C$ due to hypothesis \eqref{eqn:lipschitzFilters}, $i=1,\ldots,N$.

Using \eqref{eqn:firstTermSumAbsolute} and \eqref{eqn:secondTermSumAbsolute} back in \eqref{eqn:filterDifferenceWithXAbsolute}, and computing the norm,
\begin{align}
    & \left\| \left[ \bbH(\hbS) - \bbH(\bbS) \right] \bbx \right\|  \leq \left\| \bbD \tbx \right\|
    \label{eqn:secondOrderTermAbsolute} \\
    & \quad + \left\|
        \sum_{i=1}^{N} \tdx_{i} m_{i} h'(\lambda_{i}) \bbv_{i}
           \right\|
        \label{eqn:firstOrderTermAbsoluteFirstTerm} \\
    & \quad + \left\|   
        \sum_{i=1}^{N} 
            \tdx_{i} \bbV \diag(\bbg_{i}) \bbV^{\Hr} \bbE_{U} \bbv_{i}
           \right\|
        \label{eqn:firstOrderTermAbsoluteSecondTerm} .
\end{align}
For \eqref{eqn:firstOrderTermAbsoluteFirstTerm} we have
\begin{equation}
    \left\| \sum_{i=1}^{N} \tdx_{i} m_{i} h'(\lambda_{i}) \bbv_{i} \right\|^{2}
        = \sum_{i=1}^{N} |\tdx_{i}|^{2} |m_{i}|^{2} |h'(\lambda_{i})|^{2} \| \bbv_{i} \|^{2}
\end{equation}
since $\{\bbv_{i}\}$ conform an orthonormal basis. Then, we recall that $\|\bbv_{i}\|^{2} = 1$ and, from hypothesis \eqref{eqn:hypothesisEabsolute} we have $|m_{i}| \leq \varepsilon$ and from hypothesis \eqref{eqn:lipschitzFilters}, $|h'(\lambda_{i})| \leq C$, so that
\begin{equation}
    \left\| \sum_{i=1}^{N} \tdx_{i} m_{i} h'(\lambda_{i}) \bbv_{i} \right\|^{2}
        \leq \varepsilon^{2} C^{2} \sum_{i=1}^{N} |\tdx_{i}|^{2}.
\end{equation}
Recalling that $\sum_{i=1}^{N} |\tdx_{i}|^{2} = \| \tbx\|^{2} = \|\bbx\|^{2}$ and applying square root, we finally bound \eqref{eqn:firstOrderTermAbsoluteFirstTerm} by
\begin{equation} \label{eqn:firstOrderTermAbsoluteFirstTermBound}
    \left\| \sum_{i=1}^{N} \tdx_{i} m_{i} h'(\lambda_{i}) \bbv_{i} \right\|
        \leq \varepsilon C \| \bbx\|.
\end{equation}
Now, moving on to \eqref{eqn:firstOrderTermAbsoluteSecondTerm} and using triangle inequality together with submultiplicativity of the operator norm, we have
\begin{equation} \label{eqn:firstOrderTermAbsoluteSecondTermFirstInequality}
\begin{aligned}
    & \left\|   
        \sum_{i=1}^{N} \tdx_{i} \bbV \diag(\bbg_{i}) \bbV^{\Hr} \bbE_{U}\bbv_{i} 
    \right\| \\
        & \qquad \leq 
        \sum_{i=1}^{N}
            |\tdx_{i} |
            \| \bbV \diag(\bbg_{i}) \bbV^{\Hr}\|
            \|\bbE_{U}\|
            \|\bbv_{i}\|. 
        \\
\end{aligned}
\end{equation}
We have $\| \bbV \diag(\bbg_{i}) \bbV^{\Hr}\| \leq C$ for all $i=1,\ldots,N$ from \eqref{eqn:giAbsolute} in combination with hypothesis \eqref{eqn:lipschitzFilters}, and also $\|\bbv_{i}\| = 1$. As for $\|\bbE_{U}\|$, we know from Lemma~\ref{l:Evi} that $\|\bbE_{U}\| \leq \varepsilon \delta$. Then, 
\begin{equation} \label{eqn:firstOrderTermAbsoluteSecondTermBound}
    \left\|   
        \sum_{i=1}^{N} \bbV \diag(\bbg_{i}) \bbV^{\Hr}  \bbE_{U} (\tdx_{i} \bbv_{i}) 
    \right\|
    \leq
       C \varepsilon \delta \sqrt{N} \| \bbx \|
\end{equation}
where we used that $\sum_{i=1}^{N} |\tdx_{i} | = \| \tbx_{i} \| _{1} \leq \sqrt{N} \| \tbx \| = \sqrt{N} \| \bbx\|$.

Finally, for the second order term \eqref{eqn:secondOrderTermAbsolute} stemming from the expansion of $\hbS^{k}$, we obtain 
\begin{equation} \label{eqn:secondOrderTermAbsoluteBound}
    \| \bbD \tbx \|
        \leq \ccalO(\| \bbE \|^{2}) \| \bbx \|_{2} 
        \leq \ccalO(\varepsilon^{2}) \| \bbx \|_{2}.
\end{equation}
Using bound \eqref{eqn:firstOrderTermAbsoluteFirstTermBound} in \eqref{eqn:firstOrderTermAbsoluteFirstTerm} and bound \eqref{eqn:firstOrderTermAbsoluteSecondTermBound} in \eqref{eqn:firstOrderTermAbsoluteSecondTerm}, together with the bound \eqref{eqn:secondOrderTermAbsoluteBound} we just obtained for \eqref{eqn:secondOrderTermAbsolute}, we obtain
\begin{equation} \nonumber
    \left\| [\bbH(\hbS) - \bbH(\bbs)] \bbx \right\|
        \leq \varepsilon C \| \bbx \| + \varepsilon C \delta \sqrt{N} \| \bbx \| + \ccalO(\varepsilon^{2}) \| \bbx \|.
\end{equation}
We complete the proof by using that $\| \bbx \| = 1$ as per Def.~\ref{def_operator_distance} and recalling that we have assumed that $\bbI$ is the permutation that achieves the minimum norm of all $\bbP \in \ccalP$.
\end{proof}


\section{Stability under Relative Perturbations}\label{sec_apx_C}

\begin{proof}[Proof of Thm.~\ref{thm:filterStabilityRelative}]
Following from the fact that graph filters are permutation equivariant (Prop.~\ref{prop:filterPermutationEquivariance}), we can assume, without loss of generality, that $\bbP_{0} = \bbI$ solves \eqref{eqn_relative_perturbation_distance}. From a first order expansion analogous to \eqref{eqn:matrixTaylor}, where we use $\bbE \bbS + \bbS \bbE$ instead of just $\bbE$ as the second term, we obtain
\begin{equation} \label{eqn:matrixTaylorRelative}
\begin{aligned}
    & \bbH(\hbS) - \bbH(\bbS) \\
    & \quad = 
        \sum_{k=0}^{\infty} h_{k} 
            \sum_{r=0}^{k-1} \left( 
                \bbS^{r} \bbE \bbS^{k-r} 
                \!\! + \! \bbS^{r+1} \bbE \bbS^{k-r-1} 
            \right)
        + \bbD
\end{aligned}
\end{equation}
with $\bbD$ such that $\|\bbD\| = \ccalO(\|\bbE\|^{2})$, in analogy to \eqref{eqn:firstOrderAbsolute}.

Next, we consider the difference in the effects of the filter on an arbitrary graph signal $\bbx$ with finite energy $\|\bbx\|<\infty$ that has a GFT given by $\tbx = [\tdx_{1},\ldots,\tdx_{N}]^{\Tr}$ so that $\bbx = \sum_{i=1}^{N} \tdx_{i} \bbv_{i}$ for $\{\bbv_{i}\}_{i=1}^{N}$ the eigenvector basis of the GSO $\bbS$. Then,
\begin{align} \label{eqn:filterDifferenceWithX}
    & \left[ \bbH(\hbS) - \bbH(\bbS) \right] \bbx 
        = \sum_{i=1}^{N} \tdx_{i} \bbD \bbv_{i} 
    \\
    & \quad + \sum_{i=1}^{N} \tdx_{i} 
        \sum_{k=0}^{\infty} h_{k} 
            \sum_{r=0}^{k-1} \left( 
                \bbS^{r} \bbE \bbS^{k-r} + \bbS^{r+1} \bbE \bbS^{k-r-1} 
            \right) \bbv_{i}. 
    \nonumber
\end{align}
Let us consider first the product $\bbS^{r+1} \bbE \bbS^{k-r-1} \bbv_{i}$ in \eqref{eqn:filterDifferenceWithX}. It is immediate that $\bbS^{k-r-1}\bbv_{i} = \lambda_{i}^{k-r-1}\bbv_{i}$, and, in combination with Lemma~\ref{l:Evi}, we get
\begin{equation}
\begin{aligned}
\bbS^{r+1} \bbE \bbS^{k-r-1} \bbv_{i} 
    & = \lambda_{i}^{k-r-1} \bbS^{r+1} 
        \left(m_{i} \bbv_{i} + \bbE_{U} \bbv_{i} \right) \\
    & = m_{i} \lambda_{i}^{k} \bbv_{i}
        + \lambda_{i}^{k-r-1}\bbS^{r+1} \bbE_{U} \bbv_{i}.
\end{aligned}
\end{equation}
Analogously, for the second product, we get $\bbS^{r} \bbE \bbS^{k-r} \bbv_{i} = m_{i} \lambda_{i}^{k} \bbv_{i} + \lambda_{i}^{k-r}\bbS^{r} \bbE_{U} \bbv_{i}$.
Then, using these results, we can write
\begin{equation} \label{eqn:secondTermRelative}
\begin{aligned}
    & \sum_{i=1}^{N} \tdx_{i} \sum_{k=0}^{\infty} h_{k} \sum_{r=0}^{k-1} (\bbS^{r} \bbE \bbS^{k-r} + \bbS^{r+1} \bbE \bbS^{k-r-1}) \bbv_{i} \\
    & = 2 \sum_{i=1}^{N} \tdx_{i} m_{i} \lambda_{i} h'(\lambda_{i}) \bbv_{i} + \sum_{i=1}^{N} \tdx_{i} \bbV \diag(\bbg_{i}) \bbV^{\Hr} \bbE_{U} \bbv_{i}
\end{aligned}
\end{equation}
where, for the first term, we gathered the two equal terms $m_{i} \lambda_{i}^{k} \bbv_{i}$ and used the fact that $\sum_{k=0}^{\infty} h_{k} \lambda_{i}^{k} = \lambda_{i} h'(\lambda_{i})$; and for the second term, we defined $\bbg_{i} \in \reals^{N}$ as
\begin{equation}
\begin{aligned}[]
 [\bbg_{i}]_{j}
    & = \sum_{k=0}^{\infty} h_{k} \sum_{r=0}^{k-1} 
        \left( \lambda_{i}^{k-r-1} [\bbLambda^{r+1}]_{j} 
             + \lambda_{i}^{k-r} [\bbLambda^{r}]_{j} \right) \\
    & = \begin{cases}
        \lambda_{i} h'(\lambda_{i}) & \text{if }i=j \\
        \frac{\lambda_{i}+\lambda_{j}}{\lambda_{i} - \lambda_{j}} \left( h(\lambda_{i}) - h(\lambda_{j}) \right) & \text{if } i \neq j
      \end{cases}.
\end{aligned}
\end{equation}

Finally, we proceed to bound $\|(\bbH(\hbS) - \bbH(\bbS)) \bbx\|$. For the first term in \eqref{eqn:filterDifferenceWithX}, we simply have $\|\bbD \bbx\| \leq \ccalO(\varepsilon^{2}) \| \bbx\|$ by definition of operator norm and the error of the first order approximation \eqref{eqn:matrixTaylorRelative}. For the second term in \eqref{eqn:filterDifferenceWithX} we need to bound the two terms in \eqref{eqn:secondTermRelative}. The first of the terms in \eqref{eqn:secondTermRelative} is bounded analogously to \eqref{eqn:firstOrderTermAbsoluteFirstTermBound}, noting that, in this case, $|m_{i}| \leq \varepsilon$ by means of \eqref{eqn:hypothesisErelative}, and $|\lambda_{i} h'(\lambda_{i})| \leq C$ due to \eqref{eqn:integralLipschitzFilters}. For the second term in \eqref{eqn:secondTermRelative}, we proceed analogously to \eqref{eqn:firstOrderTermAbsoluteSecondTermBound}, where now $\|\bbE_{U}\| \leq \varepsilon \delta$ and $\| \bbV \diag(\bbg_{i}) \bbV^{\Hr}\| \leq 2C$, following the condition imposed by integral Lipschitz filters \eqref{eqn:integralLipschitzFilters}. All of these results together yield
\begin{equation} \nonumber
    \big\| ( \bbH(\hbS) - \bbH(\bbS) ) \bbx \big\|
        \leq 
        2 C \varepsilon \|\bbx\| 
        + 2C \varepsilon \delta \sqrt{N} \| \bbx\| 
        + \ccalO(\varepsilon^{2}) \| \bbx \|.
\end{equation}
We complete the proof by using that $\|\bbx\| = 1$ as per Def.~\ref{def_operator_distance}, and recalling that we have assumed that $\bbI$ is the permutation that achieves the minimum norm of all $\bbP \in \ccalP$.
\end{proof}


\begin{proof}[Proof of Thm.~\ref{thm_filter_stability_structural_constraint}]
The proof is analogous to that of Thm.~\ref{thm:filterStabilityRelative}, with the following main difference. Denote by $m_{i}$, $i=1,\ldots,N$, the eigenvalues of $\bbE = \bbU \bbM \bbU^{\Hr}$. If we order these eigenvalues as $|m_{1}| \leq \cdots \leq |m_{N}|$, we know that $\| \bbE \| = |m_{N}|$ and condition \eqref{eqn:structuralConstraint} becomes equivalent to $\| \bbE/m_{N} - \bbI \| \leq \varepsilon$. This can be used to write $\bbE \bbv_{i}$, not as in Lemma~\ref{l:Evi}, but as
\begin{equation}
    \bbE \bbv_{i} 
        = \sum_{n=1}^{N} m_{n} \bbu_{n} \bbu_{n}^{\Hr} \bbv_{i} 
        = m_{N} \sum_{n=1}^{N} 
           (1+\delta_{n}) \bbu_{n} \bbu_{n}^{\Hr} \bbv_{i}.
\end{equation}
where $m_{n}/m_{N} = 1+\delta_{n}$ for all $n = 1,\ldots,N$ with $|\delta_{n}| \leq \varepsilon$ in virtue of \eqref{eqn:structuralConstraint}, yielding
\begin{equation} \label{eqn:Evi}
    \bbE \bbv_{i} 
        = m_{N} \bbv_{i} + m_{N} \bbw_{i} 
    \quad , \quad 
    \bbw_{i} 
        = \sum_{n=1}^{N} \delta_{n} \bbu_{n} \bbu_{n}^{\Hr} \bbv_{i}.
\end{equation}
Using this expression in \eqref{eqn:secondTermRelative}, it becomes
\begin{align} \label{eqn:secondTermStructural}
& \sum_{i=1}^{N} \tdx_{i} \sum_{k=0}^{\infty} h_{k} \sum_{r=0}^{k-1} (\bbS^{r} \bbE \bbS^{k-r} + \bbS^{r+1} \bbE \bbS^{k-r-1}) \bbv_{i} \\
& = 2 m_{N} \sum_{i=1}^{N} \tdx_{i} \lambda_{i} h'(\lambda_{i}) \bbv_{i} + m_{N} \sum_{i=1}^{N} \tdx_{i} \bbV \diag(\bbg_{i}) \bbV^{\Hr} \bbw_{i}. \nonumber
\end{align}
Noting that $|m_{N}| \leq \varepsilon$ because \eqref{eqn:hypothesisErelative} holds, and that
\begin{equation} \label{eqn:wiBound}
    \| \bbw_{i} \|
        \leq \left\| \sum_{n=1}^{N} \delta_{n} \bbu_{n} \bbu_{n}^{\Hr} \right\|
                  \| \bbv_{i} \|
    = \max_{n=1,\ldots,N} |\delta_{n}| 
        \leq \varepsilon
\end{equation}
we observe that the first term of \eqref{eqn:secondTermStructural} is bounded above by $2C \varepsilon$ while the second term is bounded by $2\varepsilon^{2} C \sqrt{N} = \ccalO(\varepsilon^{2})$, completing the proof.
\end{proof}


\section{Permutation Equivariance of GNNs}\label{sec_apx_D}

\begin{proof}[Proof of Prop.~\ref{prop:GNNPermutationEquivariance}]
First, we obtain the output $\hbz_{\ell}^{fg}$ of \eqref{eqn_gnn_filter} when the input is $\hbx_{\ell-1}^{g} = \bbP^{\Tr} \bbx_{\ell-1}^{g}$, operating on the correspondingly permuted GSO $\hbS = \bbP^{\Tr} \bbS \bbP$. Since we know that application of graph filters is permutation equivariant (Prop.~\ref{prop:filterPermutationEquivariance}), we have that the output $\hbz_{\ell}^{fg}$ is
\begin{equation} \label{eqn:GNNfilterPermutation}
    \hbz_{\ell}^{fg} 
        \! = \bbH_{\ell}^{fg}(\hbS) \hbx_{\ell-1}^{g} 
        \! = \bbP^{\Tr} \bbH_{\ell}^{fg}(\bbS) \bbx_{\ell-1}^{g}
        \! = \bbP^{\Tr} \bbz_{\ell}^{fg}.
\end{equation}
Next, we note that, for any pointwise function $\sigma$ applied to each element of a vector, it holds that $\sigma(\hbx) = \sigma(\bbP^{\Tr} \bbx) = \bbP^{\Tr} \sigma(\bbx)$. Therefore, it is immediate that
\begin{equation} \label{eqn:GNNnonlinearityPermutation}
    \hbx_{\ell}^{f} 
        = \sigma \bigg[ \sum_{g}\hbz_{\ell}^{fg} \bigg]
        = \sigma \bigg[ \sum_{g} \bbP^{\Tr} \bbz_{\ell}^{fg} \bigg]
        = \bbP^{\Tr} \sigma \bigg[ \sum_{g} \bbz_{\ell}^{fg} \bigg]
\end{equation}
where we note that the last equality is $\bbP^{\Tr} \bbx_{\ell}^{f}$ using \eqref{eqn_gnn_nonlinearity}.
Since \eqref{eqn:GNNfilterPermutation}-\eqref{eqn:GNNnonlinearityPermutation} hold, we have that each layer $\ell$ of the GNN is permutation equivariant. And noting that this holds for any $\ell=1,\ldots,L$ completes the proof.
\end{proof}


\section{Graph Neural Networks Stability}\label{sec_apx_E}

\begin{proof}[Proof of Thm.~\ref{thm:GNNStability}]
Let us consider the general case where we have $F_{\ell}$ features per layer, for $\ell=0,\ldots,L$, where $F_{0}$ are the number of input features and $F_{L}$ the number of output features. That is, the input to the GNN is the collection of $F_{0}$ graph signals $\bbx = \{\bbx^{g}\}_{g=1}^{F_{0}}$ and the output $\Phi(\bbS, \bbx)$ is a collection of $F_{L}$ graph signals $\{\bbx_{L}^{f}\}_{f=1}^{F_{L}}$ [cf. \eqref{eqn_gnn_operator}]. In this context, we are interested in the difference between the output of the GNNs when evaluated on different shift operators $\bbS$ and $\hbS$
\begin{equation} \label{eqn:ineqyGNN}
    \| \Phi(\bbS, \bbx) - \Phi(\hbS, \bbx) \|^{2} 
        = \sum_{f=1}^{F_{L}} \| \bbx_{L}^{f} - \hbx_{L}^{f} \|^{2}.
\end{equation}
We use $\hat{\cdot}$ to denote an operation acting on $\hbS$ instead of $\bbS$. For example, we denote $\bbH_{\ell}^{fg}(\bbS) = \bbH_{\ell}^{fg}$ and $\bbH_{\ell}^{fg}(\hbS) = \hbH_{\ell}^{fg}$ for two filters with the same coefficients $\bbh_{\ell}^{fg}$ but acting on different shift operators $\bbS$ and $\hbS$ [cf. \eqref{eqn:graphFilter}-\eqref{eqn:graphFilter_perturbed}]. Now, focusing on one of the features of the last layer [cf. \eqref{eqn_gnn_filter}, \eqref{eqn_gnn_nonlinearity}, \eqref{eqn_gnn_operator}]
\begin{equation}
\begin{aligned}
    & \| \bbx_{L}^{f} - \hbx_{L}^{f} \|
    \\
    & = \left\| 
        \sigma \left( 
            \sum_{g=1}^{F_{L-1}} \bbH_{L}^{fg} \bbx_{L-1}^{g}
         \right) 
         - \sigma \left( 
            \sum_{g=1}^{F_{L-1}} \hbH_{L}^{fg} \hbx_{L-1}^{g}
        \right) 
        \right\|
\end{aligned}
\end{equation}
and applying Lipschitz continuity of the nonlinearity \eqref{eqn_Lipschitz_nonlinearity} by which $|\sigma(b) - \sigma(a)| \leq C_{\sigma} |b-a|$ with $C_{\sigma}=1$, followed by the triangular inequality, we get
\begin{equation} \label{eqn:diffLf}
    \| \bbx_{L}^{f} - \hbx_{L}^{f} \|
        \leq
    C_{\sigma}
    \sum_{g=1}^{F_{L-1}} 
        \left\| \bbH_{L}^{fg} \bbx_{L-1}^{g}
            - \hbH_{L}^{fg} \hbx_{L-1}^{g} \right\|.
\end{equation}
Adding and subtracting $\hbH_{L}^{fg} \bbx_{L-1}^{g}$ from the terms in the sum, and using the triangular inequality once more, we get
\begin{align}
    & \left\| \bbH_{L}^{fg} \bbx_{L-1}^{g} 
        - \hbH_{L}^{fg} \hbx_{L-1}^{g} \right\| \\
    & \quad \leq \left\| 
        \left(\bbH_{L}^{fg} - \hbH_{L}^{fg} \right) 
            \bbx_{L-1}^{g}
        \right\| 
    + \left\| \hbH_{L}^{fg}
        \left( \bbx_{L-1}^{g} - \hbx_{L-1}^{g} \right) 
      \right\|.
    \nonumber
\end{align}
The definition of operator norm, implies that
\begin{equation}
\begin{aligned}
    & \left\| 
        \bbH_{L}^{fg} \bbx_{L-1}^{g}
        - \hbH_{L}^{fg} \hbx_{L-1}^{g} 
      \right\|
    \label{eqn:diffLfg}
    \\
    & \leq 
      \left\| \bbH_{L}^{fg} - \hbH_{L}^{fg} \right\|
      \left\| \bbx_{L-1}^{g} \right\| 
    + 
      \left\| \hbH_{L}^{fg} \right\|
      \left\| \bbx_{L-1}^{g} - \hbx_{L-1}^{g} \right\|.
\end{aligned}
\end{equation}
For the first term in the inequality \eqref{eqn:diffLfg} we can use the hypothesis that $\| \bbH_{l}^{fg} - \hbH_{l}^{fg} \| \leq \Delta \varepsilon$ for all layers $l=1,\ldots,L$, while for the second term, we can use that $\|\hbH_{l}^{fg}\| \leq B=1$ for all layers. Using these two facts in \eqref{eqn:diffLfg} and substituting back in \eqref{eqn:diffLf},
\begin{equation} \label{eqn:boundL}
    \| \bbx_{L}^{f} - \hbx_{L}^{f} \|_{2} 
        \! \leq \!
    C_{\sigma} \!\!\!\!
    \sum_{g=1}^{F_{L-1}} \!\!\!
        \left( 
            \Delta \varepsilon \| \bbx_{L-1}^{g} \|
            \! + \! B \| \bbx_{L-1}^{g} - \hbx_{L-1}^{g} \|
        \right).
\end{equation}

We observe that \eqref{eqn:boundL} shows a recursion, where the bound at layer $L$ depends on the bound at layer $L-1$ as well as the norm of the features at layer $L-1$, summed over all features. That is, for an arbitrary layer $\ell \in \{ 1,\ldots,L\}$, we have
\begin{equation} \label{eqn:boundell}
    \| \bbx_{\ell}^{f} - \hbx_{\ell}^{f} \|
        \leq 
    C_{\sigma} \!\!\!
    \sum_{g=1}^{F_{\ell-1}} 
        \left( 
            \Delta \varepsilon \| \bbx_{\ell-1}^{g} \|
            + B \| \bbx_{\ell-1}^{g} - \hbx_{\ell-1}^{g} \|
        \right).
\end{equation}
with initial conditions given by the input features $\bbx_{0}^{g} = \bbx^{g}$ for $g=1,\ldots,F_{0}$, so that $\| \bbx_{0}^{g} - \hbx_{0}^{g}\| = \|\bbx^{g}- \bbx^{g} \| = 0$. For the first step to solve the recursion \eqref{eqn:boundell}, we compute the norm $\|\bbx_{\ell}^{f}\|$. We observe that
\begin{equation} \label{eqn:recursionNorm}
    \| \bbx_{\ell}^{f} \| 
        \leq C_{\sigma} \bigg\| 
            \sum_{g=1}^{F_{l-1}} 
                \bbH_{\ell}^{fg} \bbx_{\ell-1}^{g}
            \bigg\|
       \leq C_{\sigma}B  \sum_{g=1}^{F_{l-1}} \| \bbx_{\ell-1}^{g} \|
\end{equation}
where we used the triangle inequality, followed by the bound on the filters. Solving \eqref{eqn:recursionNorm} with initial condition $\|\bbx_{0}^{g}\| = \| \bbx^{g}\|$,
\begin{equation} \label{eqn:solutionNorm}
    \| \bbx_{\ell}^{f} \|
        \leq (C_{\sigma}B)^{\ell} 
            \prod_{\ell'=1}^{\ell-1} F_{\ell'} 
                \sum_{g=1}^{F_{0}} \| \bbx^{g} \|.
\end{equation}
Using \eqref{eqn:solutionNorm} back in recursion \eqref{eqn:boundell} and solving it with the corresponding initial conditions, we get
\begin{equation} \nonumber 
    \big\| \bbx_{\ell}^{f} - \hbx_{\ell}^{f} \big\|
        \leq \Delta \varepsilon (C_{\sigma}B)^{\ell-1}  
            \left(\sum_{\ell'=1}^{\ell} 
                C_{\sigma}^{\ell'} \right)
            \left(\prod_{\ell'=1}^{\ell-1}F_{\ell'} \right)
                    \sum_{g=1}^{F_{0}} \| \bbx^{g} \|.
\end{equation}
Evaluating this for $\ell=L$ and using it back in \eqref{eqn:boundL}, \eqref{eqn:ineqyGNN} yields
\begin{equation} \label{eqn:finalBoundL}
\begin{aligned}
    & \big\| \Phi(\bbS, \bbx) - \Phi(\hbS,\bbx) \big\|^{2} 
        = \sum_{f=1}^{F_{L}} \left\| 
            \bbx_{L}^{f} - \hbx_{L}^{f}
          \right\|^{2} 
    \\
    & \quad \leq 
        \sum_{f=1}^{F_{L}} \left( 
            \Delta \varepsilon (C_{\sigma}B)^{L-1}
            \sum_{\ell=1}^{L} C_{\sigma}^{\ell} 
                \prod_{\ell=1}^{L-1}F_{\ell} 
            \sum_{g=1}^{F_{0}} \| \bbx^{g} \|
        \right)^{2}.
\end{aligned}
\end{equation}
Noting that no term in the sum of \eqref{eqn:finalBoundL} depends on $f$, and subsequently applying a square root, we get
\begin{equation}
\begin{aligned}
    \big\| \Phi(\bbS, & \bbx) - \Phi(\hbS,\bbx) \big\| \\
         & \leq
    \sqrt{F_{L}} 
    \Delta \varepsilon (C_{\sigma}B)^{L-1}
    \sum_{\ell=1}^{L} C_{\sigma}^{\ell}
    \prod_{\ell=1}^{L-1}F_{\ell}
    \sum_{g=1}^{F_{0}} \| \bbx^{g} \|.
\end{aligned}
\end{equation}
Finally, setting $F_{L} = F_{0} = 1$ yields $\sqrt{F_{L}}= 1$ and $\sum_{g=1}^{F_{0}}\| \bbx^{g} \| = \| \bbx \|$, setting $F_{1} = \cdots = F_{L-1} = F$, $B=1$ and $C_{\sigma} = 1$ yields $B^{L-1} = 1$ and $\sum_{\ell=1}^{L} C_{\sigma}^{\ell} = L$, respectively. This completes the proof.
\end{proof}